%% file: arxiv.tex
\documentclass[11pt]{article}
\usepackage[utf8]{inputenc}
\usepackage[T1]{fontenc}
\usepackage{amsmath,amssymb,amsthm}
\usepackage{graphicx}
\usepackage{booktabs}
\usepackage{longtable}
\usepackage{authblk}
\usepackage[numbers,sort&compress]{natbib}
\usepackage[hidelinks]{hyperref}
\usepackage{xcolor}
\graphicspath{{figures/}}

\newcommand{\ER}{\mathrm{ER}}
\newcommand{\ERn}{\ER_{\mathrm{naive}}}
\newcommand{\ERs}{\ER_{\mathrm{split}}}
\newcommand{\ERt}{\ER_{\mathrm{truth}}}
\newcommand{\IG}{I(G)}
\newcommand{\VG}{V(G)}
\newcommand{\Vinf}{V_\infty}
\newcommand{\E}{\mathbb{E}}
\newcommand{\R}{\mathbb{R}}
\renewcommand{\P}{\mathbb{P}}
\newcommand{\eps}{\varepsilon_\infty}

\theoremstyle{plain}
\newtheorem*{propzero}{Proposition 0}
\newtheorem*{propzeroprime}{Proposition 0$'$}
\newtheorem*{propone}{Proposition 1}
\newtheorem{theorem}{Theorem}
\newtheorem*{lemmazL}{Lemma 0L}
\newtheorem*{lemmaC}{Lemma C}
\newtheorem*{principleD}{Principle D}
\theoremstyle{definition}
\newtheorem{definition}{Definition}
\newtheorem{assumption}{Assumption}

\title{\bfseries Synthetic minority data is redundant or invalid:\\
a data-dependent validity theory and a de-biased test}
\author[1]{Ahmad B. Hassanat\thanks{Corresponding author: ahmad.hassanat@gmail.com}}
\author[1]{Ahmad S. Tarawneh}
\author[2]{Ghada A. Altarawneh}
\affil[1]{Faculty of Information Technology, Mutah University, Karak, Jordan}
\affil[2]{Department of Accounting, Mutah University, Karak, Jordan}
\date{}

\begin{document}

\maketitle

\begin{abstract}
\noindent
For two decades, the standard remedy for class-imbalanced learning has been to
fabricate synthetic minority examples, and the standard evidence of their validity has
been a check that cannot fail: synthetic points are scored against the very data that
generated them. We de-bias the check. Validity becomes a population quantity---the
probability that a synthetic point truly belongs to the minority class---with a
consistent estimator that scores synthetic points against withheld real data. Where
held-out ground truth is available, the classical test underestimates true invalidity
in 96--99\% of method-by-imbalance-ratio cells, while the de-biased estimator tracks it
closely. We prove validity
is a property of the \emph{data}, not the method: class overlap sets an invalidity floor
no faithful generator escapes, making oversampling redundant where classes separate and
invalid where they overlap. Across 91 methods, three classifiers, and datasets spanning
medicine and finance---including a generator engineered to pass the classical
check---none clears both bars: gains over the best trivial baseline are noise-thin
(median below 0.01 F1, a decision threshold's reach), and most damage calibration. We release the audit as a \texttt{pip}-installable test and flip the burden
of proof: synthetic minority data must now demonstrate, on the data at hand, both
validity and information gain.
\end{abstract}

\section{Introduction}

A hospital deploys a model to flag a rare, lethal post-surgical complication. The
complication occurs in one patient in fifty, so before training, the team reaches for
the default remedy for class imbalance: a synthetic oversampler that manufactures
artificial ``patients'' until the classes balance. The practice is ubiquitous---SMOTE
\citep{chawla2002smote} and its descendants number well over a hundred variants, cited
across medicine, fraud detection, genomics and ecology \citep{fernandez2018smote15}---and
it feels safe, because the synthetic points can be checked: score each one against the
training data, confirm that its nearest real neighbours are minority patients, and
certify the augmentation as valid.

The trouble is that this certificate is printed by the party it certifies. The synthetic
points are interpolations \emph{between} real minority cases, so their nearest neighbours
in the training sample are, with high probability, their own parents---and a parent votes
``minority'' by construction. The check is not measuring whether synthetic points
represent the minority \emph{population}; it is measuring whether interpolation lies near
its own endpoints, which is a tautology. We show this parent leakage is a systematic,
sample-size-independent bias (Lemma~0L); beneath it lies an epistemic limit. In the worst case
no statistic of the training sample alone can falsify the claim that a synthetic point is
minority---a two-point Le Cam construction exhibits a data-independent generator unfalsifiable
at any sample size (Proposition~0). That is an existence result about worst-case generators;
what convicts \emph{interpolators} concretely is the parent-leakage bias of Lemma~0L. Either
way, synthetic data cannot certify itself: falsification needs withheld reference data.

The repair is as old as statistics: withhold. Define validity as a population
quantity---$\VG$, the expected probability that a point drawn from generator $G$ truly
belongs to the minority class---and estimate its complement by scoring synthetic points
against real data the generator has never seen. The resulting estimator, $\ERs$, is
consistent (Proposition~1), costs one line of protocol, and changes the verdict
everywhere. On three clinical datasets from our own earlier studies, the classical check
certifies 30, 14 and 19 of ninety-one oversamplers as valid; the de-biased check
certifies \textbf{none of the three}. And where reality itself is available to consult---datasets rich
enough that we can hide real minority cases and treat them as ground truth---the
de-biased estimator tracks true invalidity closely (dataset-level correlation 0.79--0.93;
per-ratio cells 0.47--0.98, weakest near class balance; mean error 0.03--0.05), while the
classical test underestimates it in \textbf{96--99\% of cells}
(method $\times$ imbalance-ratio combinations), by
a margin that grows exactly where the stakes do: as imbalance becomes extreme
(Fig.~\ref{fig:schematic}, Fig.~\ref{fig:invalidity}).

\emph{What ``invalid'' means.} The word carries a precise, deliberately narrow sense: a
synthetic point is invalid when it falls in the Bayes-majority region \emph{at the deployment
prior}, however faithfully it mimics a real minority case---a statement about the operating
class balance, not about realism. This is exactly why the criterion bites, and why it convicts
even real data: as an identity control makes concrete (\S\ref{sec:datadriven}), \emph{verbatim
real minority}, scored the same way, is itself judged invalid at the same floor rate, because
under class overlap the minority-typical region is genuinely majority-dominated. The
invalidity is a fact about the data's geometry and prior, not an artefact of fabrication---and
it is precisely the cost a classifier pays when overlap-region synthetic points are trusted as
minority.

Why does a fabricated patient look valid to the sample yet prove majority in truth?
Because validity was never the method's property to lose---it is the data's property to
grant. Our central result (Principle~D with Theorem~1) decomposes any generator's
invalidity into a \emph{data-driven floor} and a \emph{method-driven excess}:
$1-\VG = [1-\Vinf(P)] + [\Vinf(P)-\VG]$. The floor is set by the overlap of the true
class distributions: where minority and majority intermingle, a fraction of genuinely
minority-typical space is nonetheless majority-dominated, and any generator faithful to
the minority distribution---SMOTE in its large-sample limit \citep{sakho2024rebalancing},
or a perfect sampler of the minority density itself---inherits that fraction as
irreducible invalidity. A method escapes the floor only by \emph{abandoning} the overlap
region, which forfeits faithfulness and, as we show, any information gain. The dilemma
has two horns and no exit: where classes separate the floor vanishes but a plain
classifier already succeeds, so oversampling is \emph{redundant}; where classes overlap
the classifier needs help exactly where synthetic points are \emph{invalid}. In a
controlled sweep of class separation, the trajectory of every oversampler hugs these two
axes and never enters the valid-and-informative corner (Fig.~\ref{fig:horns}).

The prescription follows from the accounting. At the model level, appending synthetic
minority of density $g$ is \emph{exactly} class-weighting plus a geometric error term
$r(x)$ that vanishes only when $g$ equals the true minority density (Lemma~C)---so
everything recoverable about oversampling is available for free from a class weight or a
decision threshold, and everything else is unrecoverable distortion that no scalar
recalibration can remove. The empirical record agrees, and it is
classifier-independent: across 91 methods, logistic regression, XGBoost and LightGBM,
honest protocols show median F1 gains over the best trivial baseline near zero, PR-AUC and
ROC-AUC unmoved, and calibration \emph{worsened} in 77--93\% of cases. The apparent wins
reported for two decades are, on an honest test, threshold shifts in
disguise---reproducible for free, at no calibration cost, by moving the threshold
(Fig.~\ref{fig:illusion}).

We close the case from both ends. \emph{Data-side}: on credit-card fraud---a new domain,
a new feature geometry, imbalance up to 450:1---the same gap, the same tracking of
held-out truth, the same redundancy (Fig.~\ref{fig:fraud}). \emph{Method-side}: we audit
SMOTEFUNA \citep{tarawneh2020smotefuna}, an oversampler from our own group whose
acceptance rule \textbf{is} the classical validity check, built into the generator. It
passes that check perfectly---$\ERn=0.000$ by construction---and the de-biased test finds
up to 88\% of its certified ``minority'' to be majority in truth
(Fig.~\ref{fig:smotefuna}). A criterion that a method can saturate while failing is a
gamed criterion \citep{strathern1997improving}; the withheld-data estimator is the
version that cannot be gamed \citep{adebayo2018sanity}.

This paper therefore asserts a standard, not a prohibition: a synthetic-minority method
earns its place on a given dataset only by demonstrating \textbf{validity}
($\ERs\to 0$, against withheld data) \textbf{and information gain} (performance no trivial
baseline matches), and we release \texttt{resample-audit}, a documented,
\texttt{pip}-installable instrument that reports both for any resampler on any dataset in
one call. Because validity is data-dependent, it cannot be assumed to transfer from
benchmark to deployment---it must be measured in place. The burden of proof flips from
those who question synthetic minority data to those who deploy it.

Our contributions: \textbf{(i)} a population validity functional $\VG$ with a consistent,
de-biased estimator, validated against held-out ground truth (\S\ref{sec:functional});
\textbf{(ii)} an impossibility account---unfalsifiability of self-certification (Prop.~0),
the data-driven floor $1-\Vinf$ (Thm.~1, Principle~D), and the contamination identity
(Lemma~C) unifying validity, redundancy and calibration harm in one frame that also
explains prior findings \citep{sakho2024rebalancing,elreedy2024smote,lyu2025bias,goorbergh2022harm,ahmad2025concentration};
\textbf{(iii)} the largest honest audit to date---91 methods, three classifiers, medical,
financial and high-dimensional data, plus deep generative oversamplers---with both
capstones (a new domain; a check-gaming method) (\S\ref{sec:infogain}--\S\ref{sec:capstones});
\textbf{(iv)} a released, DOI-archived diagnostic tool and full reproduction pack
(\S\ref{sec:tool}). Earlier instruments from our group
\citep{tarawneh2022stop,hassanat2023jeopardy} already withheld the majority side of the
reference; this work completes the de-biasing on the minority side---a continuation, not
a recantation, of that program.

\begin{figure}[t]
\centering
\includegraphics[width=\textwidth]{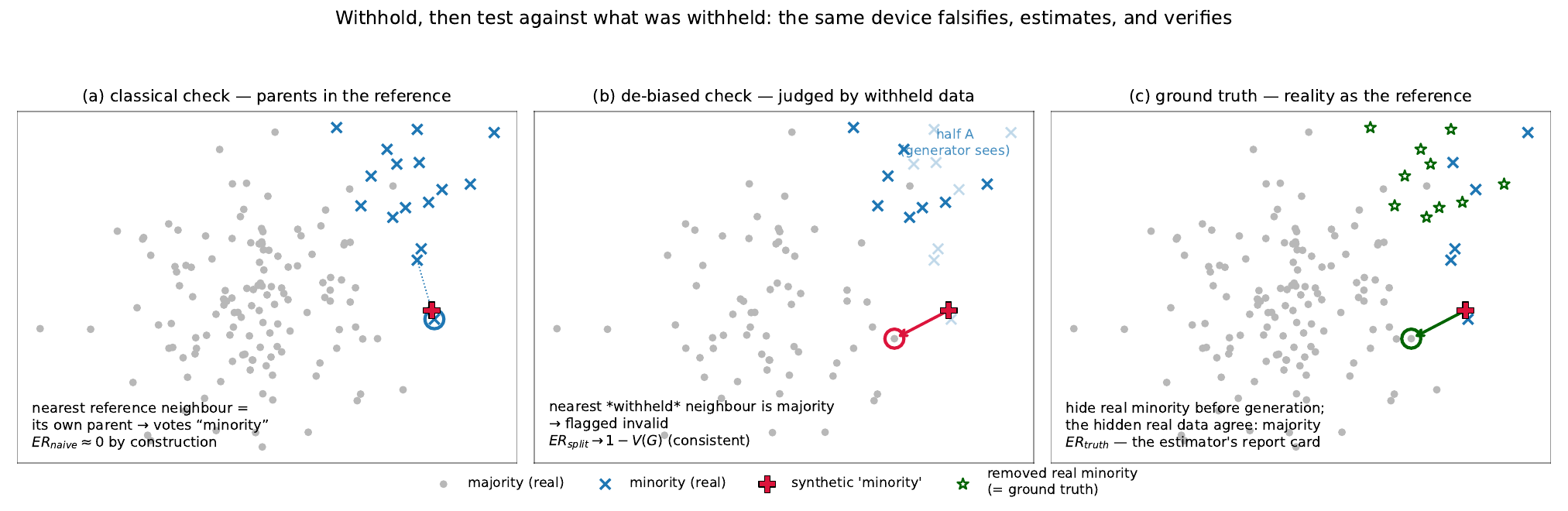}
\caption{\textbf{Withhold, then test against what was withheld.}
(a)~The classical check retains the synthetic point's own parents in the reference, so
its nearest real neighbour is a parent that votes ``minority'' by construction
($\ERn\approx 0$). (b)~The de-biased check generates on one half of each class and scores
against the withheld half, where the nearest real neighbour is majority ($\ERs$ is
consistent for $1-\VG$). (c)~The ground-truth design hides real minority before
generation and scores against it. Geometry is computed, not drawn.}
\label{fig:schematic}
\end{figure}

\section{A validity functional and a de-biased estimator}
\label{sec:functional}

\paragraph{Validity as a population claim.}
Let the data law $P$ have minority and majority class-conditional densities $f_1,f_0$
with priors $\pi_1,\pi_0$, and let $\eta(x)=P(Y{=}1\mid X{=}x)$ be the true posterior. A
generator $G$ that emits synthetic ``minority'' points $\tilde X\sim g$ makes a
population claim about each: \emph{this point is more likely minority than majority},
i.e.\ $\eta(\tilde X)\ge\tfrac12$. We define the \textbf{synthetic validity} of $G$ as
$\VG = \mathbb{E}_{\tilde X\sim g}[\eta(\tilde X)]$ (Definition~1), and its
\textbf{information gain} $\IG$ as the improvement in downstream performance over the best
\emph{trivial} alternative---no resampling, class-weighting, or a tuned decision
threshold---under an honest protocol that resamples only training data (Definition~2).
The standard this paper asserts is a conjunction: a generator earns its place on a dataset
only if $\VG\to 1$ \emph{and} $\IG>0$ there.

\paragraph{Why the classical check cannot fail (Fig.~\ref{fig:schematic}).}
The field's default check computes, for each synthetic point, the class of its nearest
real neighbour \emph{in the training sample that generated it}, and reports the
majority-vote rate---$\ERn$. Interpolating generators place each synthetic point between
two of its own minority parents, which remain in the reference. We show (Lemma~0L) that
for SMOTE-type generators the nearest reference neighbour of a synthetic point is one of
its own parents with probability bounded below \emph{uniformly in the sample size}: parent
and independent-reference distances shrink at the same $n^{-1/d}$ rate, so the leakage
never washes out, and a parent votes ``minority'' by construction. The result is a
non-vanishing downward bias---$\ERn$ converges not to $1-\VG$ but to a quantity smaller by a
constant factor. Empirically the effect dwarfs this worst-case bound---$\ERn$ near $0.000$
against $\ERs$ near $0.8$ on overlapping clinical data---so Lemma~0L, which only certifies a
positive capture probability $p^\star>0$, understates the bias it names. This is bias, not
noise; no amount of data repairs it.

The deeper obstruction is informational. Proposition~0 formalizes it as a two-point
argument in the style of Le Cam \citep{tsybakov2009nonparametric}: for any statistic $T$
computed from the training sample and the synthetic points, there exist two populations,
identical except on a region the sample is unlikely to have probed, whose induced
distributions of $T$ are within total variation $n\varepsilon$ of each other, yet whose
true validities differ maximally. This is a worst-case existence result---the separating
generator is data-independent, so it shows self-certification is unfalsifiable \emph{in
principle}; the interpolator-specific conviction is instead the parent-leakage bias of
Lemma~0L. Either way, \textbf{synthetic data cannot certify itself}---falsifiability requires reference data the generator has
never seen, and scales with the amount of withheld real data and nothing else
(Proposition~0$'$).

\paragraph{The de-biased estimator.}
Split each class at random into halves A and B; run the generator on A only; score each
synthetic point by the class of its nearest neighbour \emph{in B}---real data the
generator never saw---and average over splits:
$\ERs = \Pr(\text{nearest B-neighbour of }\tilde X\text{ is majority})$. Proposition~1
shows $\ERs$ is consistent for $1-\VG$: the nearest-neighbour vote converges to a
Bernoulli draw of $\eta(\tilde X)$ \citep{coverhart1967nn}, so its mean converges to
$\mathbb{E}[1-\eta(\tilde X)]$. We report Wilson intervals throughout, use a scale-free
distance by default \citep{hassanat2014dimensionality}, and provide a $k$-nearest-neighbour
variant for high-dimensional regimes (\S\ref{sec:scope}). Withholding, not the particular
vote, is the substance: our earlier instruments
\citep{tarawneh2022stop,hassanat2023jeopardy} already hid a share of the \emph{majority}
from the generator; $\ERs$ completes the same de-biasing on the minority side, and the
ground-truth design below closes the triangle by hiding minority. One epistemic
device---withhold, then test against what was withheld---in three roles: falsify,
estimate, verify.

\paragraph{Verification against reality.}
An estimator of a population quantity should be checked against the population, and
imbalance offers a natural way: take a dataset rich in real minority, \emph{remove} most
of it to manufacture the target imbalance, run the generators on the depleted data, and
score their synthetic points against the removed---genuinely real, genuinely
unseen---minority. Call the resulting oracle $\ERt$. Across four datasets and 1{,}040
method-by-imbalance cells (Fig.~\ref{fig:invalidity}; fraud in \S\ref{sec:capstones}):

\begin{table}[h]
\centering
\small
\setlength{\tabcolsep}{5pt}
\begin{tabular}{l cc cc c}
\toprule
& \multicolumn{2}{c}{MAE vs truth} & \multicolumn{2}{c}{corr.\ with truth} & classical \\
\cmidrule(lr){2-3}\cmidrule(lr){4-5}
dataset (cells) & classical & de-biased & classical & de-biased & underestimates \\
\midrule
spambase (383)          & 0.318 & \textbf{0.052} & 0.67 & \textbf{0.91} & 96\% \\
MagicTelescope (310)    & 0.357 & \textbf{0.039} & 0.43 & \textbf{0.93} & 99\% \\
Pima diabetes (159)     & 0.592 & \textbf{0.036} & 0.21 & \textbf{0.79} & 97\% \\
credit-card fraud (188) & 0.133 & \textbf{0.037} & 0.79 & \textbf{0.90} & 96\% \\
\bottomrule
\end{tabular}
\caption{The de-biased estimator ($\ERs$) tracks held-out real invalidity; the classical
test ($\ERn$) underestimates it in 96--99\% of cells. MAE and correlation are against the
oracle $\ERt$.}
\label{tab:groundtruth}
\end{table}

The de-biased estimator tracks held-out reality to within a few hundredths and ranks
methods almost as truth does; the classical test's shortfall \emph{grows with the
imbalance ratio} (Fig.~\ref{fig:invalidity})---blindness is worst exactly where
oversampling is reached for. A closed-form control confirms the estimator itself is not
the artifact: in a two-Gaussian model where $1-\Vinf$ is analytic, $\ERs$ tracks the exact
curve within $\pm0.011$ at every imbalance from 1:1 to 100:1 and at every reference size,
while $\ERn$ undershoots by a margin growing to 0.22 (Fig.~\ref{fig:rho}).

\paragraph{The verdict on real data.}
Re-auditing the clinical datasets of our earlier medical study under both protocols: the
classical check certifies 30 of 91 oversamplers on Pima diabetes, 14 on thoracic surgery,
19 on Framingham; the de-biased check certifies \textbf{0, 0 and 0}, with median
invalidity 0.49, 0.83 and 0.79. Methods the 2022 protocol scores as \emph{perfect}
($\ERn=0.000$) carry de-biased invalidity of 0.58--0.86 on clinical data. The gap between
the two columns is Lemma~0L made visible: the invalidity the field's standard check
structurally cannot see.

\section{Validity is data-driven}
\label{sec:datadriven}

\paragraph{The decomposition.}
Write $\Vinf(P)=\mathbb{E}_{X\sim f_1}[\eta(X)]$ for the validity of a \emph{perfect}
minority sampler---a generator that draws from the true minority density itself. Then any
generator's invalidity splits identically into two parts (Principle~D):
\begin{equation}
1-\VG \;=\; \underbrace{[\,1-\Vinf(P)\,]}_{\text{data-driven floor}}
\;+\; \underbrace{[\,\Vinf(P)-\VG\,]}_{\text{method-driven gap}}.
\label{eq:principleD}
\end{equation}
The first term depends only on the data law: the minority-typical probability mass that
falls in majority-dominated territory. Under class overlap---a positive fraction $\eps$
of minority draws land where the likelihood ratio favours the majority---the floor is
strictly positive, $1-\Vinf\ge\eps/2>0$ (Theorem~1), and it is \emph{independent of sample
size}: a limit, hence bias, not variance. Validity is therefore not a certificate a method
carries between datasets; it is granted, or refused, by the data.

\paragraph{Why the floor is real and is the last layer standing.}
Theorem~1 leans on recent limit theory rather than fighting it. \citet{sakho2024rebalancing}
prove SMOTE's law converges to $f_1$ itself as the minority sample grows;
\citet{lyu2025bias} bound its finite-sample interpolation bias and show it \emph{decays}.
Stacked, the picture has three layers: a transient interpolation bias that vanishes (Lyu),
convergence to the true minority density (Sakho), and---our contribution---the floor
$1-\Vinf>0$ that the true minority density \emph{itself} carries under overlap. The first
two wash out; the third cannot, because it belongs to $P$. It is exactly what $\ERs$
estimates, and the two-Gaussian control (Fig.~\ref{fig:rho}) displays it in closed form:
the floor rises from 0.33 to 0.95 as imbalance grows from 1:1 to 100:1 while the classical
check reports a flat near-zero. The sharpest evidence is an identity control: verbatim real
minority, scored through the same protocol, is itself judged invalid at the floor rate
($\ERs$ 0.18--0.84 across the six datasets, equal within noise to the best generators) while
the classical check calls it perfectly valid ($\ERn=0.00$)---the invalidity belongs to the
data's overlap, not to fabrication (Table~\ref{tab:realmin}).

\paragraph{No escape through the gap.}
Can a method beat the floor---achieve $\VG>\Vinf(P)$? Yes, trivially: emit points only
from the minority's safest core, far from the boundary. But that surplus is bought by
\emph{unfaithfulness} ($g\neq f_1$), and it is the wrong purchase: the abandoned overlap
region is where a classifier needed help, so information gain goes to zero even as measured
validity improves; and by Lemma~C the distortion $g\neq f_1$ is the unrecoverable part of
the model-level damage. A generator escapes the floor only by failing the other two
criteria. Validity is necessary, not sufficient---our deep-generator track makes the
dissociation concrete: variational autoencoders achieve \emph{better} validity than SMOTE
on near-separable data ($\ERs$ 0.10--0.13 versus 0.22 on sylva) and still add nothing over
a class weight.

\paragraph{Two horns, one continuum.}
Fixing imbalance at 20:1 and sweeping class separation ties the theory in one picture
(Fig.~\ref{fig:horns}). As separation grows, $\ERs$ falls monotonically from 0.79 to
0.02---the synthetic points become genuinely valid---while the baseline classifier's F1
rises to 1.0 and information gain stays $\le 0$ at \emph{every} point, reaching exactly
zero only at the valid end. In the (validity, information-gain) plane, the trajectory hugs
the axes; the valid-\emph{and}-informative corner is empty, and random oversampling traces
the same path as SMOTE---the shape belongs to the problem geometry, not to any method. Real
data supply the anchors: on \texttt{shuttle\_c0\_vs\_c4} (imbalance 14:1, separable),
SMOTE's points are \emph{perfectly} valid ($\ERs=0.000$) and perfectly useless
($\Delta$F1$=0.000$, $\Delta$ECE$=0.000$, baseline F1 0.984); on \texttt{page\_blocks0} at
comparable imbalance the same method is invalid (0.223) and calibration-harmful
($+0.11$ ECE). Same instrument, same method, opposite verdicts---chosen by the data. This
forecloses the natural rebuttal \emph{``use oversampling where its points are valid''}:
there, they are valid and buy nothing.

\paragraph{The contamination identity---the model-level ledger.}
Adding mass $w$ of density $g$ to the minority intensity shifts a calibrated learner's
target posterior by a logit gap
\begin{equation}
\delta(x)=\log\!\Big(1+\tfrac{w}{\pi_1}\tfrac{g(x)}{f_1(x)}\Big)
=\log\tfrac{\pi_1+w}{\pi_1}+r(x),
\qquad r\equiv 0 \iff g=f_1
\label{eq:lemmaC}
\end{equation}
(Lemma~C). The first term is a constant prior-shift---exactly class-weighting, removable by
a scalar offset or threshold move, ranking-invariant. The remainder $r(x)$ is geometric
distortion that no scalar correction (offset or temperature), and under mild conditions no
univariate recalibration (Platt, temperature, isotonic), can remove. The dilemma at the model level: \textbf{oversampling
$\equiv$ class-weighting $+$ an uncontrolled, unrecoverable error term}, and the term
vanishes precisely where oversampling was redundant anyway. This turns the standard
defence---``oversampling is just implicit cost-sensitive learning''---into our theorem, and
sharpens it: cost-sensitive learning is a family (loss reweighting at training time;
cost-optimal thresholds at decision time \citep{elkan2001foundations}; example-dependent
costs), and the recoverable part of oversampling equals its \emph{simplest} member---class
weighting with $a=(\pi_1{+}w)/\pi_1$---not some richer cost structure. The defence is
exactly true iff $g=f_1$, i.e.\ iff the method does nothing a class weight doesn't.

We test the identity. Fit models on oversampled data, then grant each its \emph{best
possible} scalar logit correction (an oracle upper bound on any honest recalibration).
Class-weighted models return to baseline calibration under the \emph{analytic} offset
alone, as the $g=f_1$ case predicts. Oversampled models do not: on every overlapping
dataset, 86\% remain above baseline calibration error even under the oracle, with median
residuals four to five times class-weighting's; on near-separable data the residual
collapses. The unremovable part of the harm---$r(x)$---is measurable, and it lives exactly
where the validity floor does.

\section{No information a threshold could not add}
\label{sec:infogain}

Validity asks whether the synthetic points are real minority; information gain asks whether
they help. We measure the second under an honest protocol---split first, resample the
\emph{training half only}, evaluate on untouched test data---against the three trivial
baselines any practitioner has for free: no resampling, class-weighting, and a
cross-validated decision threshold. Every comparison is per-dataset against the \emph{best}
of the three. The result is not that no method \emph{ever} edges a baseline---a handful do,
by hundredths of an F1 point---but that no gain exceeds what a decision threshold already
supplies: the margins are noise-thin (median beat below 0.01 F1), they vanish or reverse
under boosting, and they come with worse calibration. Paired with the validity floor, no
method clears both bars at once.

\paragraph{The margins are noise-thin, and the direction is consistent.}
Across six naturally-imbalanced datasets and 91 methods under logistic regression---all 91
on every set except fraud, where 76 run and 15 are omitted for runtime (SI)---the
best oversampler beats the best trivial baseline by between $+0.004$ and $+0.107$ F1, with
median beat below 0.01. A test-set bootstrap (2{,}000 resamples) settles significance: across
all six datasets and three classifiers (17 cells, Table~\ref{tab:ci}), \textbf{not one} of the
best methods has a $\Delta$F1 95\% CI excluding zero---the largest margin anywhere, Pima's
$+0.035$, carries CI $[-0.013,+0.079]$ ($p=0.08$)---and on two boosted \texttt{accidents} cells
the best margin is significantly \emph{negative}. The largest apparent ``gains'' occur only
where the no-resampling baseline degenerates outright (thoracic surgery, where unweighted
training collapses to the majority class). Under XGBoost, pooled across the five real datasets: median $\Delta$F1 $+0.0007$
versus no-resampling, median $\Delta$AUPRC $-0.0004$, median $\Delta$AUROC $0.0000$, and
65\% of pairs \emph{worsen} calibration; 21\% beat their best trivial baseline at all. Under
LightGBM: 14\%. One cell runs with the trend rather than against it---sylva under XGBoost,
where 54 of 87 methods edge no-resampling---but by at most $+0.014$ F1, on a near-separable
set whose plain model already scores F1~$0.95$, with 71 of 87 \emph{worsening} calibration and
none clearing validity: a majority of tiny, calibration-costly wins on an easy set, not new
signal. On the largest real-world set (traffic-accident severity,
$n\approx 12{,}700$; \citep{al2025interpretable}), \textbf{zero of 87 methods} beat
threshold-moving under either boosted
model. The thesis is classifier-independent.

\paragraph{The illusion, named.}
How has two decades of literature reported the opposite? Two mechanisms, both measurable.
The first is \emph{evaluation leakage}: oversampling before splitting, so near-copies of
test minority appear in training---widespread in the applied literature and repaired by a
one-line reordering. The second survives honest evaluation: F1 at the default threshold
rises while nothing about the model's \emph{ranking} improves. Across 614 method-dataset
points (Fig.~\ref{fig:illusion}; this pooled count spans every benchmarked dataset and
classifier, including the two Parkinson sets, whereas the per-method ranking tables of the
appendix report the 515 cells of the six logistic-regression datasets), median $\Delta$F1 is
$+0.06$ while median $\Delta$AUROC is
$-0.0002$ and $\Delta$AUPRC $-0.0026$; on the five real datasets the split is starker---median
$\Delta$F1 $+0.139$ against $\Delta$AUPRC $-0.005$, with \textbf{62\% of points raising F1
while PR-AUC does not improve}, and Brier worsening in 90\%. A gain in F1 with flat AUCs is
the signature of an \emph{operating-point shift}: the model did not learn more, its threshold
moved---reproducible for free, without fabricated data and without the calibration cost, by
moving the threshold. The calibration cost grows with the apparent F1 gain, exactly as
$r(x)$ predicts; in a deployed risk model a systematically inflated minority probability
\emph{is} the harm.

\paragraph{Deep generators do not change the verdict.}
Conditional VAEs, GANs, CTGAN, TVAE and a tabular diffusion model, through the identical
three-tier audit: on high-overlap data their best $\ERs$ hits the same floor as SMOTE's
($\approx 0.77$--$0.80$); on near-separable data they achieve \emph{better} validity than
any interpolator---and still no information gain (sylva: best deep F1 0.948 versus
class-weight 0.950). A better generative model solves the wrong problem: the floor and the
redundancy belong to the data.

\section{Two capstone validations}
\label{sec:capstones}

\paragraph{Data-side: credit-card fraud.}
The ULB fraud dataset \citep{dalpozzolo2015calibrating} is everything our development sets
are not: financial rather than clinical, PCA-transformed signed features, extreme natural
imbalance (578:1; capped at 102:1 for tractability---conservative, since the floor grows
with imbalance---for a working set of $n\approx 50{,}000$ rows). The findings transfer intact. The classical check certifies 56 of 69
methods; the de-biased check certifies 15, and methods scored \emph{perfect} by the
classical protocol carry $\ERs$ up to 0.58. Because fraud is largely separable here, the
plain classifier is already strong, and the redundancy horn shows in its purest form: under
both logistic regression and XGBoost the best baseline is \emph{no resampling at all}, only
3 of 75 methods exceed it, by at most $+0.005$ F1, and $\sim$95\% worsen calibration. The
ground-truth design pushes to imbalance 150--450---far beyond any benchmark in the
battery---and the de-biased estimator stays within 0.037 of held-out truth (correlation
0.90) while the classical test underestimates in 96\% of cells by a growing margin
(Fig.~\ref{fig:fraud}). A new domain, a new geometry, the same three results.

\paragraph{Method-side: a generator engineered to pass the old test.}
SMOTEFUNA \citep{tarawneh2020smotefuna}---from our own group---draws candidates in the
hyper-cuboid between a minority seed and its \emph{furthest} minority neighbour, then keeps
a candidate only if its nearest real neighbour is minority. That rule \textbf{is} the
classical validity check, moved inside the generator. It works exactly as designed: under
its own metric, $\ERn=0.000$ on every dataset tested---not approximately, but by
construction (Fig.~\ref{fig:smotefuna}). The de-biased test is unmoved: 31\% of its
certified minority is majority in truth on Pima, 49\% on Framingham, \textbf{88\% on
thoracic surgery}, and its information gain is nil ($\le +0.01$). A criterion a generator
can saturate while producing majority points is not measuring validity; it is being
gamed---Goodhart's law \citep{strathern1997improving} with an existence proof, and a
construct-validity failure the withheld-data estimator repairs. That the counterexample is
our own prior method is the point: the instrument convicts its makers' work.

Two sub-findings sharpen it. The built-in guarantee is \emph{metric-fragile}---re-scored
under a different distance, the same points show $\ERn$ of 0.23 before any
de-biasing---so ``valid'' was never a property of the synthetic data, only of a particular
test against a particular sample. And SMOTEFUNA \emph{wins} the classical comparison against
SMOTE (lower $\ERn$) while \emph{losing} the real one (higher $\ERs$ on the hardest set):
optimizing the biased criterion actively selects for hidden invalidity.

\section{A deployable standard}
\label{sec:tool}

Because validity is data-dependent (Principle~D), it cannot be certified on benchmarks and
assumed in deployment; it must be measured where the model will live. We release the
instrument as \texttt{resample-audit} (PyPI, $\ge 0.1.1$; concept DOI
10.5281/zenodo.21444930; MIT) \citep{hassanat2026resampleaudit}, a dependency-light library
whose single call \texttt{audit(X, y, resampler)} accepts anything exposing the standard
resampling interface---the SMOTE family, deep tabular generators, or a user's own
method---and returns $\ERs$ with a Wilson interval, the classical $\ERn$ for contrast, the
information gain against the three trivial baselines, ranking and calibration deltas, and a
pass/fail verdict against the conjunctive standard. The audit that fills this paper's tables
is the shipped code path, not a private pipeline. For authors of new oversamplers the
standard is constructive, not hostile: report $\VG$ and $\IG$ on your evaluation data, and
the first synthetic-minority method to clear both bars will have shown something no member
of the current family has.

\paragraph{Is the standard passable?} It is not vacuous. A method clears it by emitting
synthetic minority that is both faithful against withheld data ($\ERs\to0$) \emph{and}
informative beyond a threshold ($\IG>0$). The two bars pull in opposite directions for any
method whose only input is the observed sample: on separable data faithfulness is easy but
information is impossible (the plain model already suffices), while under overlap information
lives in the boundary region, where faithfulness fails unless the generator supplies genuinely
new knowledge of $f_1$ there---an informative prior, auxiliary data, a mechanistic model, or a
measurement the training sample lacks. That is precisely the escape our results leave open:
the standard is passable by a generator that \emph{injects real information} about the
minority law, and unpassable by interpolation or density-matching alone, because those only
re-express the sample they were given. The bar is set where a genuine advance would clear it.

\section{Scope and limitations}
\label{sec:scope}

\paragraph{Where the estimator is trustworthy---and where estimation is not.}
The de-biased estimator's ground-truth tracking (corr 0.79--0.93 at the dataset level;
0.47--0.98 per imbalance-ratio cell, weakest near balance) spans datasets from a few
hundred to fifty thousand rows. Its honest failure regime is $d\gg n$, where SMOTE is known to
degrade \citep{blaguslusa2013}: on a Spanish-speaker Parkinson's speech dataset with 100
samples and 589 features \citep{orozco2014new,hassanat2026machine},
it still recovers the \emph{level} of
invalidity five-fold better than the classical test (MAE 0.13 versus 0.63), but its
cross-method \emph{ranking} correlation collapses to 0.22. A $k$-nearest-neighbour vote
targeting the hard validity $V_{1/2}$ removes the estimator-variance part---in a matched
controlled model it restores ranking from 0.73 to 0.96, and on the real data it sharpens the
level to MAE 0.029---whereupon it exposes the deeper fact: the true hard validity is
\emph{saturated} ($0.98\pm0.04$, 93\% of methods above 0.9). In that regime there is no
ranking to recover because the entire family fails together. This is an estimation limit,
not a counterexample: at $d/n\approx6$ no distance-based quantity---including the ground
truth itself, from $\sim$12 held-out points---is reliably measurable, and the regime's
lesson (a vast, under-sampled space in which fabricated points find no real support) is the
thesis's mechanism, not its refutation.

\paragraph{Finite-reference bias on the smallest set.} $\ERs$ is asymptotic in the withheld
reference (Proposition~1); with few withheld minority points the 1-NN vote is minority-sparse
and the estimate is biased \emph{upward}---the more so the lower the true floor. We quantify
this with the identity generator on two ground-truth sets (Table~\ref{tab:finref}): at a
35-point minority reference the inflation is $+0.28$ on Pima (floor $0.49$) and $+0.53$ on
spambase (floor $0.12$), decaying to $\le0.02$ once the reference reaches a few hundred points.
Exactly one of our datasets sits in this regime---thoracic surgery, whose 70 minority leave a
35-point reference---so its absolute $\ERs$ ($0.83$) is inflated and its true floor is
uncertain (roughly $0.3$--$0.55$ after correction). Every other real and ground-truth set
withholds $\ge134$ minority (Pima 134, fraud 246, Framingham 278, sylva 402, accidents 928;
spambase and MAGIC in the thousands), where the bias is $\le0.02$, so the ground-truth tracking
(corr $0.79$--$0.93$) is measured entirely in the unbiased regime. The \emph{verdict} is
robust---even at the low end of the correction thoracic sits far above the $0.10$ certification
bar and certifies none, and the identity control inflates in lockstep, preserving the
real-versus-generator comparison---but thoracic's absolute magnitude is only an upper bound.

\paragraph{Other boundaries.}
Two methods (E\_SMOTE, ISOMAP-Hybrid) emit synthetic points in a \emph{transformed} feature
space; an original-space validity vote is undefined for them by construction, and they are
excluded by design. A small set of method-dataset cells are degenerate no-ops---findings,
not failures. One benchmark cell (ROSE $\times$ sylva $\times$ LightGBM) crashes inside
LightGBM natively on both Windows and Linux and is footnoted. Our claims concern binary
classification with tabular features; extending the validity functional to imbalanced
regression requires redefining ``invalid synthetic point'' for continuous targets---the
framework transfers, the instrument is future work.

\section{Discussion}
\label{sec:discussion}

Twenty years of synthetic oversampling rested on a check that could not fail, and the result
was a literature optimizing an unfalsifiable criterion. The repair is not a better generator;
it is a better measurement. Once validity is defined on the population and estimated against
withheld reality, the phenomenon reorganizes: overlap sets a floor no faithful generator
escapes; separability makes generation redundant; the model-level effect decomposes into free
class-weighting plus unrecoverable distortion; and the reported gains of the field resolve
into threshold shifts measured by a leaky protocol. Findings that circled this object from
different directions---the SMOTE limit law \citep{sakho2024rebalancing}, its density
\citep{elreedy2024smote}, its vanishing finite-sample bias \citep{lyu2025bias},
threshold-dominance \citep{ahmad2025concentration}, calibration harm in clinical models
\citep{goorbergh2022harm}---become facets of one account: we name the object, measure it,
bound it, and ship the measurement.

The practical prescription is estimation--decision separation, which costs nothing: estimate
honest probabilities on the data as they are; encode asymmetric costs where they belong, in
the decision rule. Scarcity is signal---the class prior is part of the data-generating
process, and a model forced to unlearn it repays the favour with inflated risk estimates
precisely for the cases that matter. Where deployment priors genuinely differ, explicit prior
correction does what fabricated data does opaquely, without the geometric residue.

We do not claim synthetic minority data can never help; we claim helping now has a test. The
standard is conjunctive, per-dataset, and public: demonstrate validity against withheld
reality, and information beyond a threshold. Nothing in 91 classical methods, five deep
generators, or a method built to pass the old check clears it. The claim is falsifiable, and
we make the target explicit: a counterexample is a single method that, on a given dataset
under a pre-registered protocol, clears \emph{both} bars at once---low $\ERs$ against withheld
reality \emph{and} a ranking gain over the best trivial baseline (positive $\Delta$PR-AUC or
$\Delta$ROC-AUC, above noise and reproducible across seeds and classifiers), not merely a
higher F1 at the default threshold. We tabulate every method's per-cell ranking deltas in the
Supplementary Information so that this search is fully auditable: across all method--dataset
cells the median ranking gain is negative, and the isolated positives are neither valid nor
reproducible. An $F1$ gain alone does not qualify, precisely because it is the operating-point
shift a threshold reproduces for free. The burden of proof has changed hands.

\section*{Methods}

\emph{Condensed; formal statements and proofs---Prop.~0/0$'$, Lemma~0L, Prop.~1, Thm.~1,
Principle~D, Lemma~C---plus full protocol, dataset provenance, environments and the
excluded-cell inventory are in the Supplementary Information.}

\paragraph{Validity protocol.} $\ERn$: 1-NN vote against the generator's own training sample,
all parents retained. $\ERs$: five random half-splits per class; generate on half~A---so the
generator is trained on $n_1/2$ minority points; because the floor is sample-size-independent
(Theorem~1) this does not bias $\ERs$ as an estimate of $1-\VG$, though it modestly raises
method-specific variance, which averaging over the five splits controls---then 1-NN
vote against half~B (reference capped at 20{,}000, all minority kept); Wilson intervals;
Hassanat scale-free distance by default \citep{hassanat2014dimensionality,hassanat2026optimality};
a $k$-NN vote (with $k$ cross-validated; the classical $\sqrt{n}$ rule is dimension-suboptimal,
\citep{hassanat2026optimality}) targets the hard validity $V_{1/2}$ for $d\gg n$.
Deterministic per-method seeding.

\paragraph{Ground truth.} Remove real minority to a target ratio; the removed points form the
$\ERt$ reference; per (method, ratio) cell report naive/split/truth.

\paragraph{Benchmark.} Stratified 70/30 split; standardize on train; resample train only;
LogReg / XGBoost / LightGBM; baselines no-resample, class-weight (per-classifier mechanism),
threshold-move (3-fold OOF F1-optimal); F1, ROC-AUC, PR-AUC, Brier, ECE. Per-cell 95\%
confidence intervals and one-sided significance for the deltas are obtained by bootstrapping
the test rows (2{,}000 resamples; \texttt{benchmark\_ci.py}).

\paragraph{Data availability.} All datasets used are publicly available: the KEEL imbalanced
benchmarks \citep{alcala2011keel}, the oversampling methods and interfaces of the
smote-variants library \citep{kovacs2019smotevariants}, the ULB credit-card fraud set
\citep{dalpozzolo2015calibrating}, and the clinical, accident-severity and Parkinson's
speech-feature datasets described in the appendix, whose original sources are cited therein.
The processed result CSVs underlying every table and figure are deposited in the paper
repository.

\paragraph{Code availability.} The audit instrument is released as \texttt{resample-audit}
(PyPI; Zenodo concept DOI 10.5281/zenodo.21444930; MIT licence). The full experimental
harness, figure scripts and reproduction protocol are in the paper repository; every table
regenerates via \texttt{repro\_tables.py} and every figure via \texttt{make\_all\_figures.py}.

\paragraph{Use of large language models.} During the preparation of this manuscript the authors
used a large language model-based assistant to support drafting and copy-editing of text,
organisation of references, scaffolding and debugging of analysis and figure code, and
formatting of tables. It was not used to generate data, experimental results, or scientific
claims; all experiments, theoretical results, analyses, and conclusions were conceived,
executed, and verified by the authors, who reviewed and edited every output and take full
responsibility for the content. The large language model does not meet authorship criteria and
is not credited as an author.

% ---- figure floats for sections 3-6 ----
\begin{figure}[t]\centering
\includegraphics[width=\textwidth]{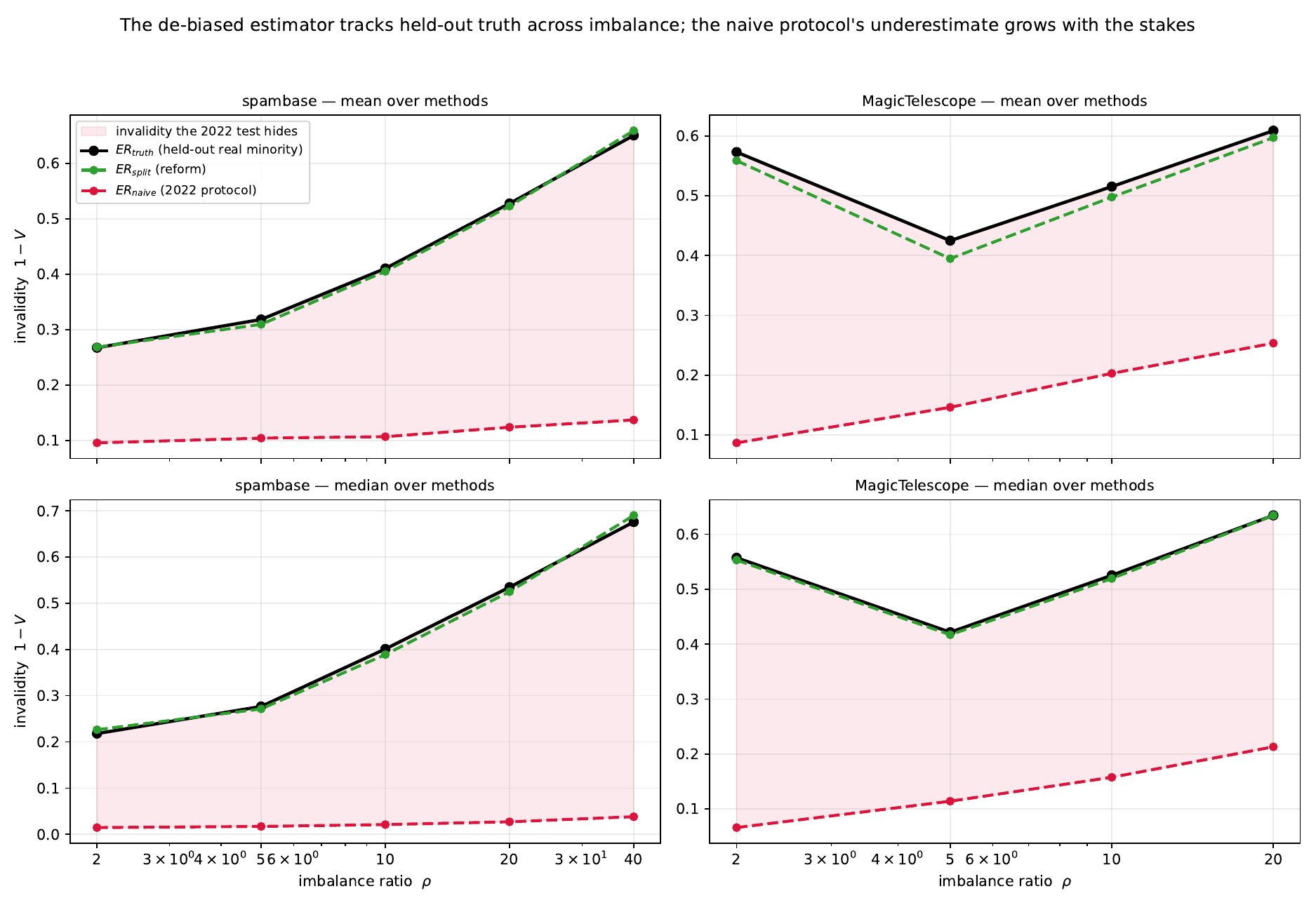}
\caption{\textbf{The de-biased estimator tracks held-out truth across imbalance; the naive
protocol's underestimate grows with the stakes.} Mean and median over methods on two
ground-truth datasets; shaded band is the invalidity the classical test hides.}
\label{fig:invalidity}\end{figure}

\begin{figure}[t]\centering
\includegraphics[width=\textwidth]{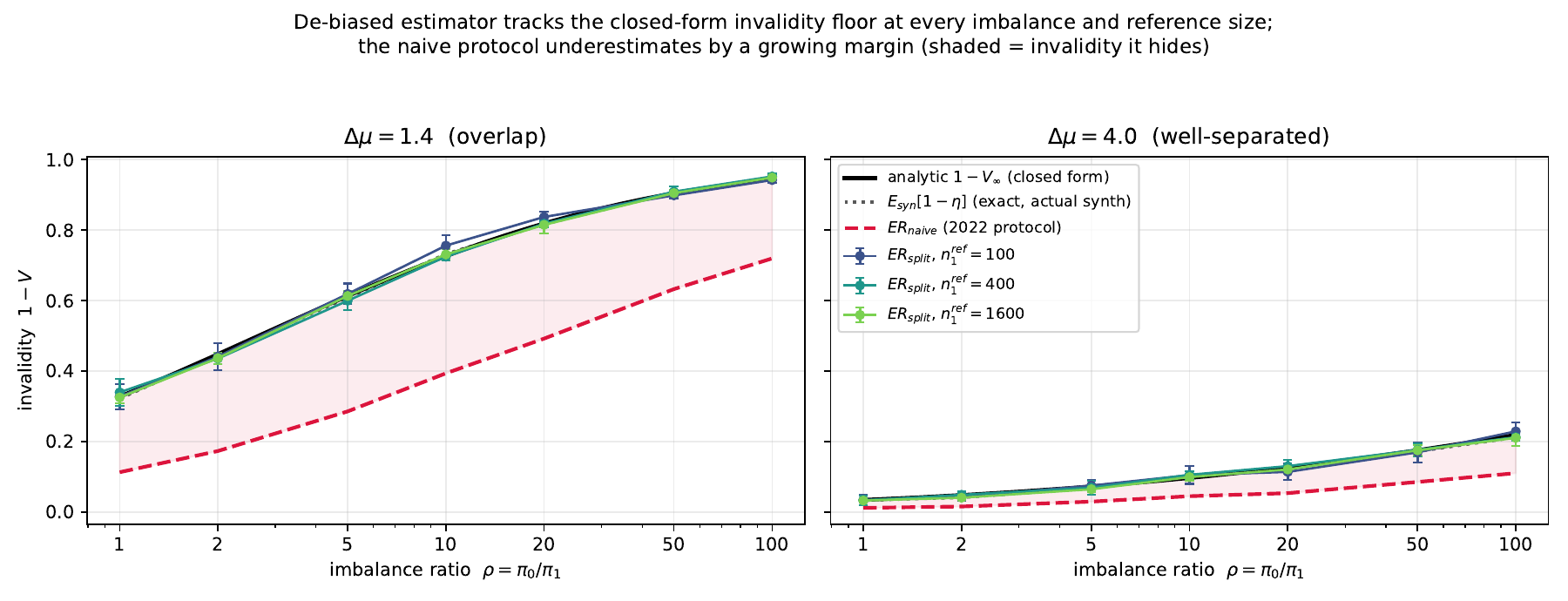}
\caption{\textbf{Closed-form control.} In a two-Gaussian model the de-biased estimator tracks
the analytic $1-\Vinf$ at every imbalance and reference size; $\ERn$ undershoots by a growing
margin.}
\label{fig:rho}\end{figure}

\begin{figure}[t]\centering
\includegraphics[width=\textwidth]{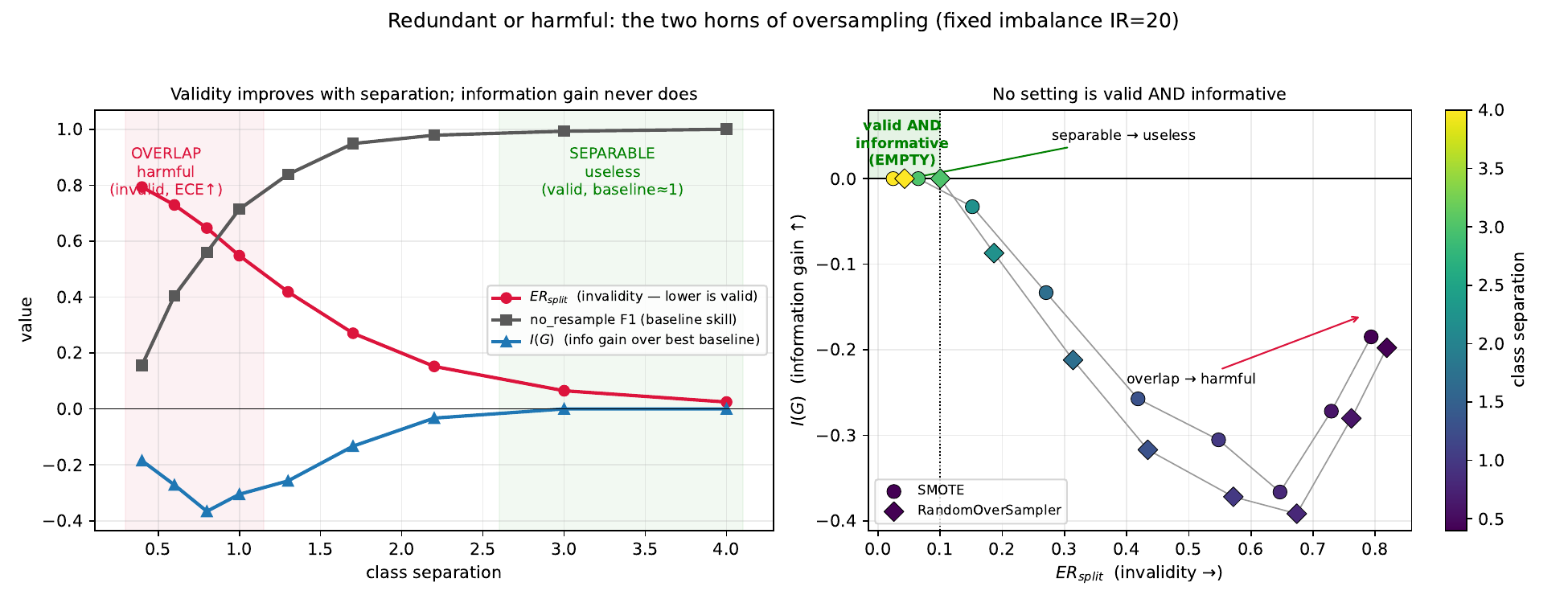}
\caption{\textbf{Redundant or harmful.} Sweeping class separation at fixed imbalance: validity
improves ($\ERs\downarrow$) while information gain never does; the (validity, information-gain)
plane has an empty valid-and-informative corner.}
\label{fig:horns}\end{figure}

\begin{figure}[t]\centering
\includegraphics[width=\textwidth]{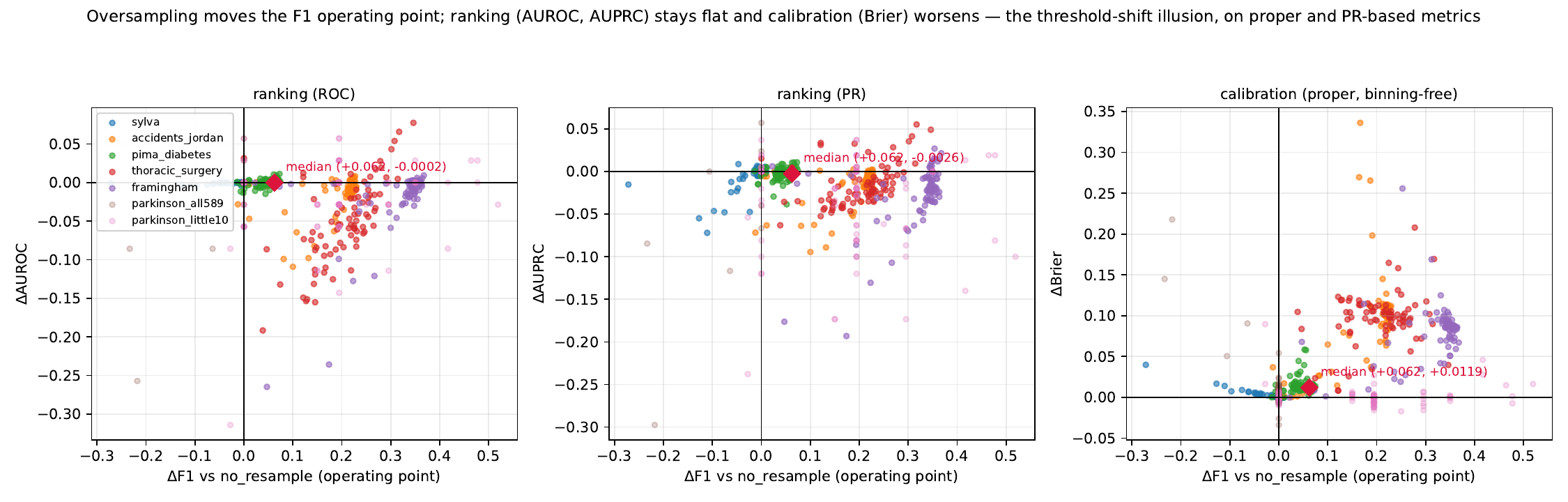}
\caption{\textbf{The threshold-shift illusion.} Oversampling moves the F1 operating point while
ranking (AUROC, AUPRC) stays flat and calibration (Brier) worsens.}
\label{fig:illusion}\end{figure}

\begin{figure}[t]\centering
\includegraphics[width=0.85\textwidth]{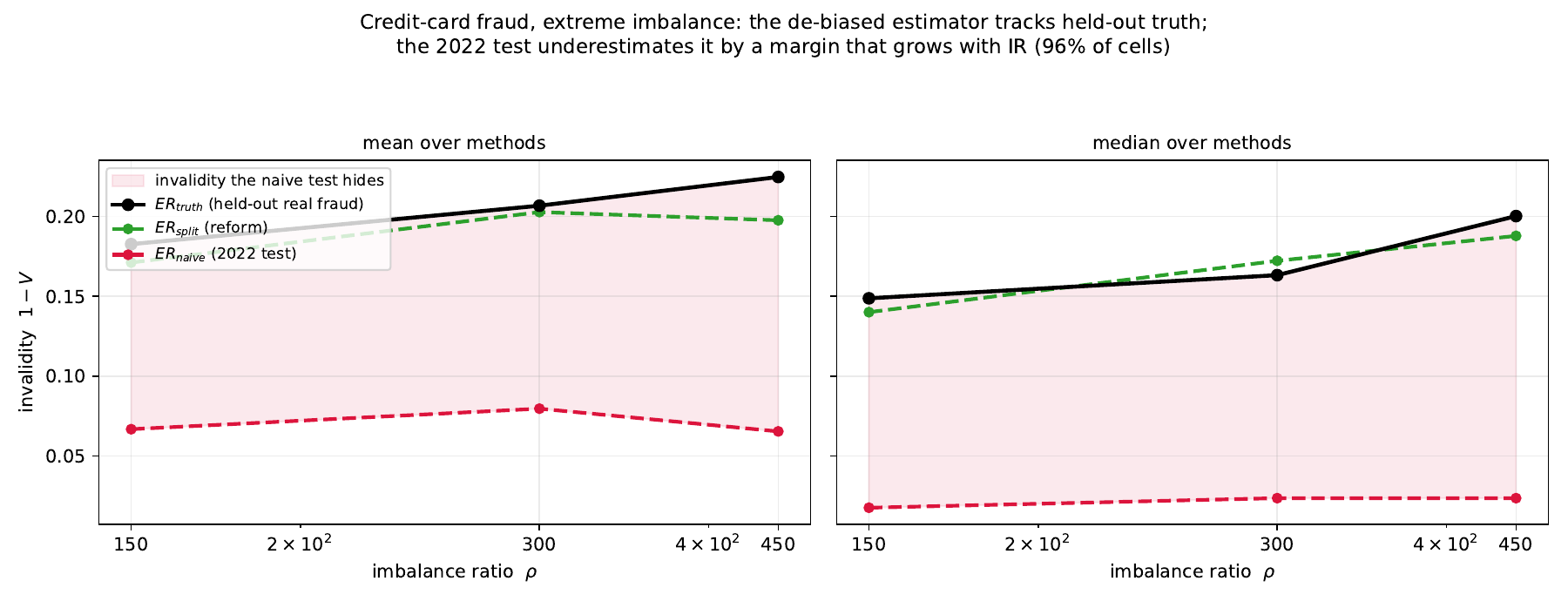}
\caption{\textbf{Data-side capstone (credit-card fraud).} At imbalance 150--450 the de-biased
estimator tracks held-out real fraud; the classical test underestimates in 96\% of cells.}
\label{fig:fraud}\end{figure}

\begin{figure}[t]\centering
\includegraphics[width=0.8\textwidth]{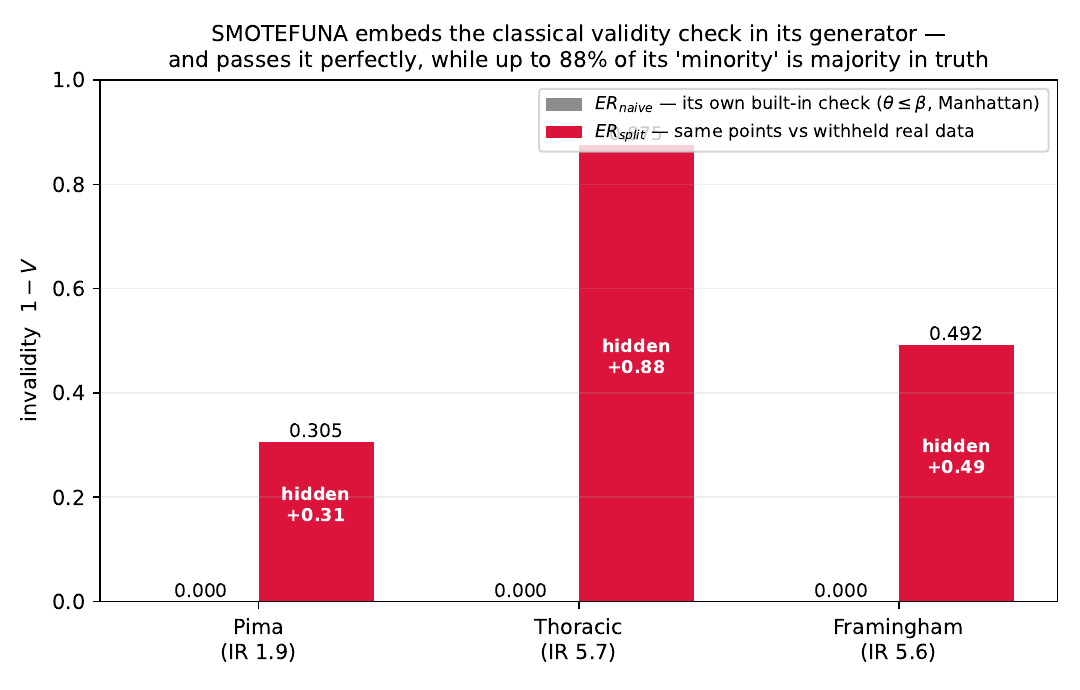}
\caption{\textbf{Method-side capstone (SMOTEFUNA).} Its acceptance rule is the classical check,
so $\ERn=0$ by construction; the de-biased test finds up to 88\% of its certified ``minority''
to be majority in truth.}
\label{fig:smotefuna}\end{figure}

\section*{Author contributions}
A.B.H. conceived the study, developed the data-dependent validity theory and the de-biased
estimator, solely authored the \texttt{resample-audit} software, and led the writing. A.B.H.
and A.S.T. designed the audit methodology and experiments, and A.S.T. implemented and executed
the benchmark, ground-truth and deep-generator experiments. G.A.A. contributed to dataset
curation and to the analysis and interpretation of the clinical and applied results. All
authors discussed the findings and reviewed, edited and approved the manuscript.

\section*{Competing interests}
The authors declare no competing interests.

\clearpage
\appendix
\setcounter{table}{0}\renewcommand{\thetable}{S\arabic{table}}
\setcounter{figure}{0}\renewcommand{\thefigure}{S\arabic{figure}}
\setcounter{equation}{0}\renewcommand{\theequation}{S\arabic{equation}}
\begin{center}\Large\bfseries Supplementary Information\end{center}
\medskip
\noindent\textit{Formal statements and proofs supporting the main text. Notation follows
\S2 of the main text. Results are labelled to
match the main text rather than by running number: Propositions~0 and~0$'$ (impossibility of
self-certification), Proposition~1 (consistency of the split estimator), Theorem~1 (the
invalidity floor), Lemma~0L (parent leakage in the classical check), Lemma~C (the
contamination identity), and Principle~D (the data-driven decomposition).}

\section{Formal framework}

\paragraph{Setup.} A population $P$ on $\R^d\times\{0,1\}$ has class priors
$\pi_1=\P(Y{=}1)$, $\pi_0=1-\pi_1$ with $\pi_1\le\pi_0$, class-conditional densities
$f_1,f_0$ w.r.t.\ a common dominating measure, and posterior
$\eta(x)=\pi_1 f_1(x)/(\pi_1 f_1(x)+\pi_0 f_0(x))$. Write $\rho=\pi_0/\pi_1\ge1$. A
generator $G$ trained on $D=\{(X_i,Y_i)\}_{i=1}^n$ produces synthetic minority points
$\tilde X_1,\dots,\tilde X_m\sim g$ (the density $g$ may depend on $D$). All
``minority/majority'' labels refer to the Bayes rule $\mathbf{1}[\eta\ge\tfrac12]$, i.e.\
to the population, never to training labels.

\begin{definition}[synthetic validity]
$\VG=\E_{\tilde X\sim g}[\eta(\tilde X)]\in[0,1]$; the invalidity is
$1-\VG=\E_{\tilde X\sim g}[1-\eta(\tilde X)]$. The hard invalidity is
$1-V_{1/2}(G)=\E_{\tilde X}[\mathbf{1}[\eta(\tilde X)<\tfrac12]]$.
\end{definition}

\begin{definition}[information gain]
For a learner $\ell$ and metric $M$, let $M^\star_{\mathrm{triv}}(D)$ be the best test-set
$M$ over $\{$no-resample, class-weight, threshold-move$\}$ fit on $D$, and $M(G,D)$ the
test-set $M$ of $\ell$ trained on $D$ augmented by $G$. Then $I(G)=M(G,D)-M^\star_{\mathrm{triv}}(D)$.
The \emph{standard} is the conjunction $\{\VG\to1\}\wedge\{I(G)>0\}$ on the dataset at hand.
\end{definition}

\section{Unfalsifiability and the repair by withholding}

\begin{propzero}[self-certification is impossible]
Fix any $n,m$ and any statistic $T=T(D,\tilde X_{1:m})$. For every $\varepsilon>0$ there
exist populations $P,P'$ and a data-independent generator $G$ such that $P,P'$ differ only
on a region $A$ with mass $\le\varepsilon$, carry opposite Bayes labels on $A$, yet
$\mathrm{TV}(\mathrm{Law}_P(T),\mathrm{Law}_{P'}(T))\le n\varepsilon$ \emph{free of $m$},
while $|V_P(G)-V_{P'}(G)|=1$.
\end{propzero}
\begin{proof}
Two-point (Le Cam) construction \citep{tsybakov2009nonparametric}. Fix a marginal law $\mu$
for $X$ and a Borel region $A$ with $\mu(A)\le\varepsilon$. Define $P$ and $P'$ to share the
marginal $\mu$ and the posterior $\eta$ off $A$, and to differ only on $A$: put
$\eta\equiv1$ on $A$ under $P$ (Bayes-minority) and $\eta\equiv0$ on $A$ under $P'$
(Bayes-majority); for $\varepsilon$ small the added minority mass keeps $\pi_1\le\pi_0$ in
both populations. Take the generator $G$ to be data-independent with law
$\mathrm{Unif}(A)$, identical under $P$ and $P'$. Then $V_P(G)=\E_{\tilde X}[\eta_P(\tilde
X)]=1$ and $V_{P'}(G)=\E_{\tilde X}[\eta_{P'}(\tilde X)]=0$, so $|V_P(G)-V_{P'}(G)|=1$.

Couple the two experiments on a common probability space: draw
$X_1,\dots,X_n\stackrel{\text{iid}}{\sim}\mu$ (shared) and
$\tilde X_{1:m}\sim\mathrm{Unif}(A)$ (shared, drawn independently of the labels). Off $A$
assign each $Y_i$ by the common conditional; on $A$ set $Y_i=1$ under $P$ and $Y_i=0$ under
$P'$. On the event $E=\{X_i\notin A\ \text{for all }i\}$ the pair $(D,\tilde X_{1:m})$---and
hence any statistic $T(D,\tilde X_{1:m})$ of it---is identical under $P$ and $P'$. The
coupling inequality gives
\[
\mathrm{TV}\big(\mathrm{Law}_P(T),\mathrm{Law}_{P'}(T)\big)\ \le\ \P(E^{c})\ \le\
\sum_{i=1}^{n}\P(X_i\in A)\ =\ n\,\mu(A)\ \le\ n\varepsilon,
\]
free of $m$, since the synthetic points are coupled identically and their law does not
depend on $P$ versus $P'$. Thus no statistic of the generator together with its own training
data can separate two populations of maximally different validity; the discriminating signal
is confined to the $O(n\varepsilon)$ probability that a training point lands in $A$.
\end{proof}

\begin{propzeroprime}[withholding restores identification]
Fix $(P,G)$ under the regularity of Proposition~1. Let $R$ be a reference
sample of size $r$ drawn from $P$ and independent of the data used to fit $G$. Then
$\ERs\xrightarrow{\ \P\ }1-\VG$ as $r\to\infty$, and in the construction above the reference
distinguishes $P$ from $P'$ once $r\gtrsim1/\varepsilon$.
\end{propzeroprime}

\begin{lemmazL}[parent retention is a non-vanishing bias]
For SMOTE with fixed $K$, the nearest neighbour of $\tilde X$ in a reference that retains
its own parents is a parent with probability $\ge p^\star>0$, with $p^\star$ depending only
on $(K,d,\pi_1)$ and local densities and bounded below \emph{uniformly in $n$}. Hence
$\ERn\to\E[(1-\eta(\tilde X))(1-p(\tilde X))]\le(1-p^\star)(1-\VG)$, a downward bias that is
bounded away from zero uniformly in $n$ (a constant-factor attenuation, not a vanishing rate).
\end{lemmazL}
\begin{proof}
Write $\tilde X=X_i+U(X_j-X_i)$ with $X_i$ a minority seed, $X_j$ one of its $K$ minority
nearest neighbours, and $U\sim\mathrm{Unif}[0,1]$; the nearer parent lies at distance
$\min(U,1-U)\,\|X_j-X_i\|\le\tfrac12\|X_j-X_i\|$. Fix a Lebesgue point $x$ of $f_0,f_1$ with
$f(x)=\pi_1 f_1(x)+\pi_0 f_0(x)>0$ and condition on $\tilde X\approx x$. Rescale space by the
factor $n^{-1/d}$ about $x$. By the local (Poisson) limit for nearest-neighbour functionals
of an i.i.d.\ sample whose density is continuous and positive at $x$ \citep{penrose2003weak}, the
rescaled training points converge in distribution to a homogeneous Poisson process $\Pi$ of
intensity $f(x)$, within which the minority points are an independent thinning of intensity
$\pi_1 f_1(x)$. In this limit the two parents $X_i,X_j$ are points of the minority
sub-process at the $K$-nearest-neighbour scale $(K/(\pi_1 f_1(x)))^{1/d}$, with $\tilde X$ on
the segment between them, while the nearest \emph{non}-parent training point is the nearest
point of $\Pi$ at scale $(1/f(x))^{1/d}$. The two scales stand in the fixed ratio
$\big(Kf(x)/(\pi_1 f_1(x))\big)^{1/d}$, independent of $n$: the parent-versus-field contest is
genuinely order one. Hence the event $\mathcal C=\{\text{a parent is the nearest of all training points to }
\tilde X\}$ has a strictly positive limiting probability
\[
p(x)\ =\ \P\big(\text{a parent beats every point of }\Pi\big)\ \in\ (0,1),
\]
a continuous functional of $(K,d)$ and the local ratio $\pi_1 f_1(x)/f(x)$; on the support
carrying $g$-mass it is bounded below by some $p^\star>0$ depending only on $(K,d,\pi_1)$ and
the local density bounds, \emph{uniformly in $n$}. On $\mathcal C$ the vote is minority (a
parent is a minority point); off $\mathcal C$ the nearest non-parent votes majority with
probability $\to 1-\eta(\tilde X)$ by nearest-neighbour consistency (as in
Proposition~1). Therefore
\[
\ERn\ \to\ \E\big[(1-p(\tilde X))\,(1-\eta(\tilde X))\big]\ \le\ (1-p^\star)\,\E[1-\eta(\tilde X)]
\ =\ (1-p^\star)(1-\VG),
\]
a downward bias of at least $p^\star(1-\VG)$: bounded away from zero uniformly in $n$ whenever
$1-\VG>0$.
\end{proof}

\begin{propone}[consistency of the split estimator]
Under (i) $\eta$ continuous, (ii) $f_0,f_1$ with common support and locally bounded density
ratios, (iii) reference size $\to\infty$ with the generator's own sample fixed per split,
$\ERs\xrightarrow{\ \P\ }1-\VG$.
\end{propone}
\begin{proof}
Condition on $\tilde X=x$. As the withheld reference $B$ grows, its nearest neighbour to
$x$ converges to $x$ a.s.\ (NN consistency), and the label of that neighbour is a Bernoulli
draw with success probability $\eta(x)$ \citep{coverhart1967nn}. Hence
$\P(\mathrm{nn}_B(x)\in\text{majority})\to1-\eta(x)$. Taking expectation over $\tilde X\sim g$
and exchanging limit and integral by dominated convergence (integrand in $[0,1]$) gives
$\E_{\tilde X}[1-\eta(\tilde X)]=1-\VG$. Averaging over splits preserves the limit.
\end{proof}

\noindent\emph{Remark (separable case).} Assumption~(ii) (common support) is used only to
place minority reference mass near synthetic points in the overlap region. Under separability
the class supports are disjoint, the overlap mass $\eps=0$, and both the floor $1-\Vinf=0$
(Theorem~1) and $\ERs\to0$ hold trivially: interpolated minority points stay in the minority
support, where their nearest withheld neighbour is minority. The separable showcases (sylva,
\texttt{shuttle\_c0\_vs\_c4}, fraud) sit in this regime, and the estimator remains consistent
by the argument above restricted to each class support.

\section{The data-driven floor}

\begin{assumption}[Overlap]
$\P_{X\sim f_1}(f_1(X)/f_0(X)<\rho)=\eps>0$; a positive fraction of minority-typical points
fall in the Bayes-majority region. Under separability $\eps=0$.
\end{assumption}

\begin{theorem}[invalidity floor]
For SMOTE with fixed $K$ and $n_1\to\infty$ (so $K/n_1\to0$), \citet[Thm.~3.2]{sakho2024rebalancing}
give convergence \emph{in distribution} of the SMOTE-generated variable $Z_{K,n_1}$ to the
minority law: $\P[Z_{K,n_1}\in B]\to\int_B f_1$ for every Borel $B$ (a convergence of the full
synthetic-point law, not of a marginal statistic). With $\eta$ bounded continuous the
portmanteau theorem then gives
$V(\mathrm{SMOTE}_{n_1})=\E[\eta(Z_{K,n_1})]\to\E_{X\sim f_1}[\eta(X)]=:\Vinf$, so
\[
1-\Vinf=\E_{X\sim f_1}[1-\eta(X)]=\int(1-\eta)f_1\,dx\ \ge\ \tfrac12\eps\ >\ 0
\quad\text{under Overlap.}
\]
\end{theorem}
\begin{proof}
The cited $Z_{K,n_1}\Rightarrow f_1$ is convergence of the full SMOTE law, so with
bounded-continuous $\eta$ the portmanteau theorem gives the stated limit directly. On
$\{f_1/f_0<\rho\}$ we have $\pi_1 f_1<\pi_0 f_0$, hence
$\eta<\tfrac12$, hence $1-\eta>\tfrac12$; integrating over that event (mass $\eps$ under
$f_1$) gives $\int(1-\eta)f_1\ge\tfrac12\eps$.
\end{proof}

\noindent The floor is (i) independent of $n_1$ (bias, not variance); (ii) inherited even by
a perfect minority sampler; (iii) exactly what $\ERs$ estimates. Three literature layers:
transient interpolation bias $\to0$ \citep{lyu2025bias}; limit is $f_1$
\citep{sakho2024rebalancing}; the non-vanishing floor (ours).

\begin{principleD}
With $\Vinf(P)=\E_{f_1}[\eta]$, $1-\VG=[1-\Vinf(P)]+[\Vinf(P)-\VG]$ (add and subtract
$\Vinf$). The first bracket depends only on $P$; the second is method-specific. Moreover
$\VG>\Vinf(P)$ requires $g\neq f_1$ with mass shifted toward $\{\eta\to1\}$: by Lemma~C the
surplus is the unrecoverable $r(x)$, and mass moved off the overlap boundary reduces $I(G)$.
A generator beats the floor only by failing faithfulness and information gain.
\end{principleD}

\section{The contamination identity}

\begin{lemmaC}[contamination identity]
Append synthetic minority of density $g$, weight $w$, to $(\pi_1 f_1,\pi_0 f_0)$ and
renormalize; a population-calibrated learner (strictly-proper-loss ERM over a rich class,
\citet{gneiting2007proper}) targets $\eta_{\mathrm{CM}}$ with:
\begin{itemize}
\item[\textbf{(C1)}] $\mathrm{logit}\,\eta_{\mathrm{CM}}(x)=\mathrm{logit}\,\eta(x)+\delta(x)$,
$\delta(x)=\log(1+\tfrac{w}{\pi_1}\tfrac{g(x)}{f_1(x)})$.
\item[\textbf{(C2)}] $\delta\equiv\mathrm{const}\iff g=f_1$, whence
$\delta\equiv\log\tfrac{\pi_1+w}{\pi_1}$: a scalar offset (ranking/ROC invariant, $\eta$
recovered exactly), equal to class-weighting with $a=(\pi_1{+}w)/\pi_1$.
\item[\textbf{(C3)}] $g\neq f_1\Rightarrow\delta$ non-constant: no scalar offset or scalar
temperature restores $\eta$ off a null set, and if $\delta$ varies along an
$\eta_{\mathrm{CM}}$-level set, no univariate recalibration---any map of $\eta_{\mathrm{CM}}$
alone, scalar offset, scalar temperature, Platt and isotonic included---restores calibration.
\item[\textbf{(C4)}] $\delta(x)=\log\tfrac{\pi_1+w}{\pi_1}+r(x)$,
$r(x)=\log\frac{\pi_1 f_1+w g}{(\pi_1+w)f_1}$, $r\equiv0\iff g=f_1$. Hence oversampling
$\equiv$ class-weighting $+$ geometric error $r(x)$.
\end{itemize}
\end{lemmaC}
\begin{proof}
(C1) Bayes' rule on the contaminated intensities; the normalizer $1/(1+w)$ cancels; subtract
logits. (C2) $\delta\equiv c\Rightarrow g/f_1\equiv(e^c-1)\pi_1/w$ on $\{f_1>0\}$, i.e.\
$g=\kappa f_1$, and $\int g=\int f_1=1$ forces $\kappa=1$; converse by substitution;
class-weight equivalence by the $a$-weighted posterior. (C3) a non-constant logit gap cannot
be closed by a constant offset, nor by any scalar temperature, since a temperature rescaling
$\mathrm{logit}\,\eta_{\mathrm{CM}}\mapsto\mathrm{logit}\,\eta_{\mathrm{CM}}/T$ is itself a
function of $\eta_{\mathrm{CM}}$ alone; more generally, if
$\eta_{\mathrm{CM}}(x_1)=\eta_{\mathrm{CM}}(x_2)$ but $\eta(x_1)\neq\eta(x_2)$, every
recalibration $\varphi(\eta_{\mathrm{CM}})$---offset, temperature, Platt or isotonic
included---returns one value at the two points and so errs at one of them by $\ge$ half
the gap. (C4) algebra from (C1)--(C2).
\end{proof}

\noindent\textbf{Empirical signature.} With empirical counts the analytic correction is
$-\log(n_1'/n_1)$ on the logit. Reporting ECE/Brier raw, after the analytic offset, and after
an oracle scalar offset (best single logit shift by test-NLL): class-weight returns to
baseline under the analytic offset; 86\% of oversampled models (5 real datasets, LogReg, 386
points) remain above baseline even under the oracle, median residual 4--5$\times$
class-weight's; the residual collapses on near-separable data---$r$ lives where the floor does.

\section{The $k$-NN estimator and hard validity $V_{1/2}$}

The default vote ($k{=}1$) estimates $\E[1-\eta]$ (Proposition~1). A $k$-NN
majority vote estimates $1-V_{1/2}=\P(\eta<\tfrac12)$ as $k\to\infty,\,k/n\to0$ (the vote
converges to the Bayes label), with lower per-point variance---recommended when $d/n$ is
large. In a controlled $d=589$ model it restores cross-method ranking from $\rho=0.73$
($k{=}1$) to $0.96$ ($k\approx11$--$21$); on real high-dimensional data it sharpens the level
(MAE $0.13\to0.03$) and reveals a saturated truth ($0.98\pm0.04$), i.e.\ uniform failure of
the family. Here $k$ is a variance-reduction knob on an \emph{estimator}, not a classifier,
and is chosen by cross-validation---the strongest $k$-selector \citep{hassanat2026optimality}.
The textbook $k\approx\sqrt{N}$ rule ($N$ the withheld-reference size voted against per split,
\emph{not} the full sample $n$) is minimax-optimal only at $d=4$ and is dimension-suboptimal
otherwise; the dimension-aware default $k\approx\lfloor N^{4/(d+4)}\rfloor$
\citep{hassanat2026optimality} is the principled non-cross-validated alternative. Our target,
however, is the cross-method \emph{ranking} of an averaged functional $\P(\eta<\tfrac12)$ over
synthetic points---not per-point $0/1$ risk---so a moderate $k$ suppresses label-vote variance
without materially disturbing that ranking: in the controlled $d{=}589$ study (a simulated
model, not the 100-row data) the ranking plateaued over $k\approx11$--$21$, while on the real
set the smaller reference shifts the useful band down.

\section{Protocols, data, and excluded cells}

\paragraph{Datasets.} Table~\ref{tab:datasets} lists every dataset used in the paper, its
domain, size, dimensionality, class split, imbalance ratio (IR $=$ majority/minority), and
source. Six real, naturally-imbalanced sets carry the validity gap and benchmark
(Tables~\ref{tab:s1},~\ref{tab:s3}): \texttt{sylva}, a forest-cover-type task from the
Agnostic-Learning-vs-Prior-Knowledge challenge \citep{guyon2007agnostic} (classifying
Ponderosa pine against all other cover types, with half of its 216 features injected
distractors); \texttt{accidents}, a traffic-accident-severity set collected by the authors
in Jordan \citep{al2025interpretable} (major vs.\ minor injury; one-hot encoded); the clinical trio \texttt{Pima}
(diabetes onset in Pima women \citep{smith1988pima}), \texttt{thoracic} (one-year
post-operative mortality after lung resection \citep{zieba2014boosted}), and
\texttt{Framingham} (ten-year coronary-heart-disease risk, the widely-circulated modelling
extract of the Framingham Heart Study \citep{mahmood2014framingham}); and \texttt{fraud},
the ULB European credit-card set \citep{dalpozzolo2015calibrating} (PCA-anonymised features,
extreme 578:1 imbalance, capped at 102:1 for tractability). Two near-balanced sets serve as
\emph{ground-truth base populations} for the hide-the-minority design of \S3---\texttt{spambase}
\citep{hopkins1999spambase} and \texttt{MAGIC} gamma telescope \citep{bock2004magic}---because
their large minority mass supports removing real minority to synthesise imbalance while
retaining a held-out oracle. The two-horn separation sweep is anchored on the KEEL imbalanced
benchmarks \texttt{shuttle\_c0\_vs\_c4} (separable) and \texttt{page\_blocks0} (overlapping)
\citep{alcala2011keel}. The high-dimensional caveat (main-text Scope section) uses a balanced $d{=}589$ feature set
(100 recordings, 50/50): the 589 acoustic features of \citet{hassanat2026machine} computed on
the Spanish Parkinson's-disease speech corpus of \citet{orozco2014new}.

\begin{table}[h]\centering\small
\caption{\textbf{Datasets.} $n$ rows after cleaning; $d$ features after any one-hot encoding;
minority/majority counts; IR $=$ majority$/$minority. Source column gives the dataset origin.}
\label{tab:datasets}
\begin{tabular}{llrrrrl}
\toprule
dataset & domain & $n$ & $d$ & min/maj & IR & source \\
\midrule
sylva & forest cover type & 13{,}086 & 212 & 805/12{,}281 & 15.3 & \citep{guyon2007agnostic} \\
accidents & road-accident severity & 12{,}669 & 104 & 1{,}856/10{,}813 & 5.8 & \citep{al2025interpretable} \\
Framingham & 10-yr CHD risk & 3{,}658 & 15 & 557/3{,}101 & 5.6 & \citep{mahmood2014framingham}\textsuperscript{a} \\
Pima & diabetes onset & 768 & 8 & 268/500 & 1.9 & \citep{smith1988pima} \\
thoracic & post-op mortality & 470 & 22 & 70/400 & 5.7 & \citep{zieba2014boosted} \\
fraud (ULB) & credit-card fraud & 284{,}807 & 30 & 492/284{,}315 & 578 & \citep{dalpozzolo2015calibrating} \\
spambase\textsuperscript{b} & spam e-mail & 4{,}601 & 57 & 1{,}813/2{,}788 & 1.5 & \citep{hopkins1999spambase} \\
MAGIC\textsuperscript{b} & gamma/hadron shower & 19{,}020 & 10 & 6{,}688/12{,}332 & 1.8 & \citep{bock2004magic} \\
Parkinson (Spanish) & PD speech (high-$d$) & 100 & 589 & 50/50 & 1.0 & \citep{orozco2014new,hassanat2026machine} \\
shuttle\_c0\_vs\_c4 & sensor (horn anchor) & 1{,}829 & 9 & 123/1{,}706 & 13.9 & KEEL \citep{alcala2011keel} \\
page\_blocks0 & doc.\ layout (horn anchor) & 5{,}472 & 10 & 559/4{,}913 & 8.8 & KEEL \citep{alcala2011keel} \\
\bottomrule
\end{tabular}
\end{table}

\noindent\textsuperscript{a}\,No single canonical citation exists for the circulated modelling
subset; we cite the study's historical reference. Counts are complete-case: 582 of the 4{,}240
rows carry missing values and are dropped.
\textsuperscript{b}\,Near-balanced; used only as a ground-truth base population (\S3), not
in the natural-imbalance benchmark.

\textbf{Environments.} Experiment CSVs used a pinned stack (Python 3.10; numpy 2.2.6,
scikit-learn 1.7.2, scipy 1.15.3, imbalanced-learn 0.14.2, smote\_variants 1.0.1). Figures
regenerate unchanged under current libraries. Historical raw rows predate the
deterministic-seeding fix and reproduce within split-std ($\pm0.01$--$0.05$); the current
harness is bit-reproducible.

\noindent\textbf{Excluded cells (cross-platform-verified).} (a)~Degenerate no-ops (no
synthetic points; ER undefined). (b)~Feature-space-transforming methods (E\_SMOTE,
ISOMAP-Hybrid) emit points in a reduced space, so an original-space vote is undefined by
construction---excluded by design. (c)~ROSE $\times$ sylva $\times$ LightGBM: native LightGBM
crash on both Windows (4.6.0) and Linux (4.7.0), clean data---a library bug. (d)~DEAGO,
MOT2LD recovered via TensorFlow and a TSNE compatibility shim for the validity run (hence a
numeric $\ERs$ in the split table), but the shim did not carry into the benchmark pipeline,
where the cell is \texttt{x}; likewise ISOMAP\_Hybrid on Pima is \texttt{deg} for validity
(empty synthetic set under the split) and \texttt{lib} at benchmark (reduced-space output).
These codes track the processing stage, not a contradiction.

\noindent\textbf{Reproduction.} \texttt{repro\_tables.py} regenerates every table with
pass/fail checks; \texttt{make\_all\_figures.py} regenerates every figure with a checkpoint
report; \texttt{reproduction\_protocol.md} maps each artifact to its command. All 1{,}040
ground-truth cells and 614 benchmark points are in the result CSVs.

\section{Extended results tables}

All values below are computed from the released result CSVs by
\texttt{code/si\_tables.py} and match the summary figures in the main text. SMOTE-family
methods only (deep generators are a separate track); ``ok'' counts exclude the
deterministic no-ops and feature-space-transforming methods of \S6.

\begin{table}[h]\centering\small
\caption{\textbf{Validity gap per dataset.} Median $\ERn$ vs median $\ERs$, and the
number of methods each protocol certifies (invalidity $<0.10$). The de-biased test
certifies almost none.}
\label{tab:s1}
\begin{tabular}{lcccccc}
\toprule
dataset & $n_{\mathrm{ok}}$ & med.\ $\ERn$ & \#cert.\ naive & med.\ $\ERs$ & \#cert.\ split & med.\ gap \\
\midrule
sylva & 78 & 0.000 & 65 & 0.216 & 1 & +0.216 \\
accidents & 78 & 0.043 & 46 & 0.845 & 0 & +0.801 \\
Framingham & 82 & 0.181 & 19 & 0.788 & 0 & +0.607 \\
Pima & 80 & 0.147 & 30 & 0.493 & 0 & +0.347 \\
thoracic & 79 & 0.221 & 14 & 0.826 & 0 & +0.606 \\
fraud & 69 & 0.031 & 56 & 0.177 & 15 & +0.146 \\
\bottomrule
\end{tabular}
\end{table}

\begin{table}[h]\centering\small
\caption{\textbf{Real minority scores at the floor (identity baseline).} Verbatim real
minority points scored through the \emph{identical} protocol (the ``generator'' is the
identity map). $\ERn$(real)$\approx0$---each real point self-matches in the parent-retaining
reference, so the classical test certifies real data as perfectly valid---yet $\ERs$(real)
sits at the floor $1-\Vinf$, matching the median generator's $\ERs$ (Table~\ref{tab:s1})
within noise. The invalidity the de-biased test reports is therefore the overlap floor that
real minority itself carries, not an artefact of fabrication.}
\label{tab:realmin}
\begin{tabular}{lccc}
\toprule
dataset & $\ERn$(real) & $\ERs$(real)\,$=$\,floor & median generator $\ERs$ \\
\midrule
sylva      & 0.00 & 0.251 & 0.216 \\
accidents  & 0.00 & 0.836 & 0.845 \\
Framingham & 0.00 & 0.794 & 0.788 \\
Pima       & 0.00 & 0.510 & 0.493 \\
thoracic   & 0.00 & 0.789 & 0.826 \\
fraud      & 0.00 & 0.176 & 0.177 \\
\bottomrule
\end{tabular}
\end{table}

\begin{table}[h]\centering\small
\caption{\textbf{Finite-reference bias of $\ERs$ (identity generator).} Real minority scored
against a withheld minority reference of decreasing size (majority reference fixed), on two
ground-truth sets with different true floors. The 1-NN estimate is biased \emph{upward} when
the withheld minority is sparse, and the inflation is larger for a lower true floor
(spambase). It decays to $\le0.02$ once the reference reaches a few hundred points. Among our
datasets only thoracic surgery (35-point reference) sits in the biased regime; all others
withhold $\ge134$ minority.}
\label{tab:finref}
\begin{tabular}{lcccc}
\toprule
dataset & true floor & bias @\,35 ref. & bias @\,100 ref. & unbiased by \\
\midrule
Pima     & 0.49 & $+0.28$ & $+0.09$ & $\sim$134 \\
spambase & 0.12 & $+0.53$ & $+0.32$ & $\sim$900 \\
\bottomrule
\end{tabular}
\end{table}

\begin{table}[h]\centering\small
\caption{\textbf{Ground truth by imbalance ratio.} Mean absolute error and correlation of
each protocol against the held-out-minority oracle $\ERt$, per (dataset, ratio) cell. The
de-biased estimator tracks truth; the classical test underestimates it in 95--100\% of these
individual (dataset, ratio) cells (96--99\% when pooled per dataset, as reported in the main
text), worsening with the ratio.}
\label{tab:s2}
\begin{tabular}{llccccc}
\toprule
dataset & IR & cells & MAE naive & MAE split & corr naive/split & naive $<$ truth \\
\midrule
spambase & 2 & 78 & 0.172 & 0.025 & 0.90 / 0.98 & 97\% \\
spambase & 5 & 78 & 0.216 & 0.049 & 0.86 / 0.81 & 96\% \\
spambase & 10 & 77 & 0.299 & 0.043 & 0.79 / 0.90 & 97\% \\
spambase & 20 & 76 & 0.404 & 0.071 & 0.73 / 0.81 & 95\% \\
spambase & 40 & 74 & 0.513 & 0.073 & 0.61 / 0.81 & 95\% \\
MagicTelescope & 2 & 73 & 0.486 & 0.033 & 0.35 / 0.95 & 100\% \\
MagicTelescope & 5 & 79 & 0.279 & 0.041 & 0.45 / 0.90 & 99\% \\
MagicTelescope & 10 & 79 & 0.316 & 0.041 & 0.53 / 0.89 & 99\% \\
MagicTelescope & 20 & 79 & 0.355 & 0.040 & 0.47 / 0.91 & 100\% \\
Pima & 2 & 78 & 0.684 & 0.034 & 0.31 / 0.47 & 97\% \\
Pima & 5 & 81 & 0.503 & 0.037 & 0.43 / 0.79 & 96\% \\
fraud & 150 & 65 & 0.116 & 0.029 & 0.84 / 0.97 & 95\% \\
fraud & 300 & 64 & 0.127 & 0.026 & 0.87 / 0.98 & 97\% \\
fraud & 450 & 59 & 0.159 & 0.057 & 0.64 / 0.73 & 97\% \\
\bottomrule
\end{tabular}
\end{table}

\begin{table}[h]\centering\small
\caption{\textbf{Information gain is classifier-independent.} Best trivial baseline (F1),
number of oversamplers beating it, the single best margin, and how many worsen calibration
(ECE), per dataset and classifier. Denominators (75--93) count the methods returning a valid
benchmark result in each cell---the SMOTE-family and deep generators, less per-cell no-ops and
failures; the fraud column is smaller owing to the runtime omissions of \S6. The best margins
are noise-thin and largest only where the unresampled baseline degenerates (thoracic); on the
largest real-world set (accidents), boosting yields \emph{zero} methods above threshold-moving.}
\label{tab:s3}
\begin{tabular}{llcccc}
\toprule
dataset & clf & best baseline (F1) & \#beat & best margin & worsen ECE \\
\midrule
sylva & LogReg & class\_weight (0.950) & 15/93 & +0.006 & 87/93 \\
sylva & XGBoost & no\_resample (0.953) & 54/87 & +0.014 & 71/87 \\
sylva & LightGBM & no\_resample (0.962) & 17/86 & +0.006 & 68/86 \\
accidents & LogReg & threshold\_move (0.317) & 25/93 & +0.009 & 83/93 \\
accidents & XGBoost & threshold\_move (0.321) & 0/87 & $-$0.040 & 23/87 \\
accidents & LightGBM & threshold\_move (0.319) & 0/87 & $-$0.026 & 17/87 \\
Pima & LogReg & threshold\_move (0.678) & 25/88 & +0.026 & 84/88 \\
Pima & XGBoost & class\_weight (0.654) & 28/88 & +0.040 & 83/88 \\
Pima & LightGBM & class\_weight (0.650) & 40/88 & +0.048 & 79/88 \\
thoracic & LogReg & class\_weight (0.239) & 26/88 & +0.107 & 79/88 \\
thoracic & XGBoost & threshold\_move (0.281) & 8/87 & +0.069 & 63/87 \\
thoracic & LightGBM & threshold\_move (0.319) & 1/87 & +0.049 & 67/87 \\
Framingham & LogReg & threshold\_move (0.422) & 2/93 & +0.004 & 90/93 \\
Framingham & XGBoost & threshold\_move (0.339) & 2/87 & +0.021 & 44/87 \\
Framingham & LightGBM & threshold\_move (0.327) & 2/87 & +0.006 & 22/87 \\
fraud & LogReg & no\_resample (0.877) & 3/75 & +0.005 & 71/75 \\
fraud & XGBoost & no\_resample (0.885) & 3/75 & +0.005 & 72/75 \\
\bottomrule
\end{tabular}
\end{table}

\begin{table}[h]\centering\small
\caption{\textbf{No benchmark margin is statistically significant.} For each dataset
$\times$ classifier cell we take the method with the largest F1 margin over the best trivial
baseline and report $\Delta$F1, its percentile bootstrap 95\% CI (2{,}000 test-set resamples),
and how many of the top-eight methods have a $\Delta$F1 CI excluding zero (``sig''). \textbf{Across all 17 cells, none do}---not one
oversampler significantly beats the best trivial baseline; on two boosted \texttt{accidents}
cells the best margin is significantly \emph{negative}. Because $\Delta$F1 is largest for the
highest-F1 method, the top eight bracket the best margin exactly, so $0/8$ significant implies
$0/91$.}
\label{tab:ci}
\begin{tabular}{llcccc}
\toprule
dataset & clf & best method & $\Delta$F1 & 95\% CI & sig \\
\midrule
Pima & LogReg & Gaussian\_SMOTE & $+0.035$ & $[-0.013,+0.079]$ & 0/8 \\
Pima & XGBoost & SVM\_balance & $+0.024$ & $[-0.050,+0.069]$ & 0/8 \\
Pima & LightGBM & AHC & $+0.047$ & $[-0.012,+0.083]$ & 0/8 \\
thoracic & LogReg & ISMOTE & $+0.082$ & $[-0.045,+0.194]$ & 0/7 \\
thoracic & XGBoost & AHC & $+0.032$ & $[-0.179,+0.183]$ & 0/8 \\
thoracic & LightGBM & Lee & $+0.059$ & $[-0.132,+0.239]$ & 0/8 \\
Framingham & LogReg & Safe\_Level\_SMOTE & $+0.001$ & $[-0.017,+0.018]$ & 0/8 \\
Framingham & XGBoost & IPADE\_ID & $+0.025$ & $[-0.008,+0.060]$ & 0/8 \\
Framingham & LightGBM & SVM\_balance & $+0.014$ & $[-0.040,+0.064]$ & 0/8 \\
sylva & LogReg & Stefanowski & $+0.006$ & $[-0.012,+0.018]$ & 0/8 \\
sylva & XGBoost & SUNDO & $+0.013$ & $[-0.001,+0.020]$ & 0/8 \\
sylva & LightGBM & SMOTE\_AMSR & $+0.006$ & $[-0.007,+0.017]$ & 0/8 \\
accidents & LogReg & SMOTE\_AMSR & $+0.013$ & $[-0.001,+0.022]$ & 0/8 \\
accidents & XGBoost & MCT & $-0.033$ & $[-0.061,-0.007]$ & 0/8 \\
accidents & LightGBM & SMOTE\_FRST\_2T & $-0.028$ & $[-0.050,-0.004]$ & 0/8 \\
fraud & LogReg & DBSMOTE & $+0.004$ & $[-0.007,+0.017]$ & 0/8 \\
fraud & XGBoost & SMOTE\_ENN & $+0.000$ & $[-0.004,+0.000]$ & 0/8 \\
\bottomrule
\end{tabular}
\end{table}

\begin{table}[h]\centering\small
\caption{\textbf{Method-side capstone (SMOTEFUNA), under its own Manhattan metric.}
$\ERn=0$ by construction (its acceptance rule is the classical check); the de-biased test
finds a large share of its certified minority to be majority in truth. $z$-scored features,
as in the original paper.}
\label{tab:s4}
\begin{tabular}{lccc}
\toprule
dataset & $\ERn$ (own metric) & $\ERs$ (de-biased) & hidden bias \\
\midrule
Pima & 0.000 & 0.305 & +0.305 \\
thoracic & 0.000 & 0.875 & +0.875 \\
Framingham & 0.000 & 0.492 & +0.492 \\
\bottomrule
\end{tabular}
\end{table}

\clearpage
\section{Full per-method tables (all 91 methods)}

Complete per-method values for every \texttt{smote\_variants} oversampler, read directly
from the released CSVs (\texttt{code/si\_full\_tables.py}). Column headers: sylva,
accid(ents), Fram(ingham), Pima, thor(acic), fraud. Deep generative oversamplers are a
separate track and are omitted here.

\paragraph{Justification of non-numeric cells.} Every blank cell has a documented,
non-arbitrary cause, and \emph{no} excluded cell is ever counted toward validity or
information gain:
\begin{itemize}
\item[\texttt{deg}] \textbf{No synthetic points under the stated protocol}: the generator
produced \emph{no} synthetic minority points under the half-split, so validity is undefined on
an empty synthetic set. Several such methods---cleaning-based hybrids like SMOTE\_ENN, and
SMOBD, SOMO, VIS\_RST---\emph{do} generate on the full data (hence a numeric $\ERn$ and a
benchmark F1) but empty under the split; \texttt{deg} thus means ``undefined under this
protocol,'' not ``adds nothing.'' Where such a method runs at benchmark time its downstream
effect is reported in the F1/$\Delta$ tables (e.g.\ SMOTE\_ENN lowers accidents F1 by $0.220$).
\item[\texttt{lib}] \textbf{Feature-space-transforming} (E\_SMOTE, ISOMAP\_Hybrid): these
emit synthetic points in a \emph{reduced} feature space (a genetic-algorithm mask,
$d{:}8{\to}5$; an Isomap embedding, $8{\to}3$), so a validity vote in the original space is
undefined \emph{by construction} (\S6). Excluded by design---the generators run without
error on both operating systems; the instrument correctly refuses the dimensionality
mismatch.
\item[\texttt{x}] \textbf{Run failure}: a numerical or library error on that dataset
(recorded verbatim in the CSV \texttt{status} column).
\item[\texttt{--}] \textbf{Omitted for runtime}: appears \emph{only} in the fraud column
(76/91 methods run; all 91 run on the other five datasets). The 15 absent methods are the
optimization-based and backend-dependent generators---SMOTE\_PSO, SMOTE\_PSOBAT, AMSCO,
ADG, GASMOTE, DSRBF, IPADE\_ID, SMOTE\_FRST\_2T, SSO, NEATER, KernelADASYN,
Supervised\_SMOTE (each of which can take hours to days at the fraud set's scale,
$n\!\approx\!50{,}000$), plus DEAGO and MOT2LD (which require an optional TensorFlow /
TSNE backend). They are omitted from the extreme-scale run for tractability and reported
as omitted, never as valid; the other five datasets audit the full family.
\end{itemize}
The one benchmark cell that crashes irreducibly (ROSE\,$\times$\,sylva\,$\times$\,LightGBM,
a native LightGBM fault on both Windows and Linux; \S6) is footnoted rather than tabulated.

\renewcommand{\arraystretch}{0.92}

{\footnotesize
\begin{longtable}{l cccccc}
\caption{\textbf{De-biased invalidity $\ERs$ per method.} The headline table: no method
achieves low $\ERs$ on the overlapping sets. Read against Table~\ref{tab:s1}
(per-dataset medians).}
\label{tab:full_split}\\
\toprule
method & sylva & accid & Fram & Pima & thor & fraud \\ \midrule
\endfirsthead
\multicolumn{7}{l}{\footnotesize\emph{Table~\ref{tab:full_split} continued}}\\
\toprule method & sylva & accid & Fram & Pima & thor & fraud \\ \midrule
\endhead
\bottomrule \multicolumn{7}{r}{\footnotesize continued\ldots}\\
\endfoot
\bottomrule
\multicolumn{7}{@{}p{0.95\textwidth}@{}}{\footnotesize\emph{Codes:} \texttt{deg} no-op (empty synthetic set), \texttt{lib} feature-space-transforming (vote undefined), \texttt{x} run failure, \texttt{--} omitted for runtime (fraud only, 15 slow/backend methods; see justification above). No excluded cell counts toward validity or gain.}\\
\endlastfoot
\input{si_full_ersplit.tex}
\end{longtable}
}

{\footnotesize
\begin{longtable}{l cccccc}
\caption{\textbf{Classical invalidity $\ERn$ per method.} The same points scored against
the parent-retaining reference; compare cell-by-cell with Table~\ref{tab:full_split} to
see the leakage bias (Lemma~0L).}
\label{tab:full_naive}\\
\toprule
method & sylva & accid & Fram & Pima & thor & fraud \\ \midrule
\endfirsthead
\multicolumn{7}{l}{\footnotesize\emph{Table~\ref{tab:full_naive} continued}}\\
\toprule method & sylva & accid & Fram & Pima & thor & fraud \\ \midrule
\endhead
\bottomrule \multicolumn{7}{r}{\footnotesize continued\ldots}\\
\endfoot
\bottomrule
\multicolumn{7}{@{}p{0.95\textwidth}@{}}{\footnotesize\emph{Codes:} \texttt{deg} no-op (empty synthetic set), \texttt{lib} feature-space-transforming (vote undefined), \texttt{x} run failure, \texttt{--} omitted for runtime (fraud only, 15 slow/backend methods; see justification above). No excluded cell counts toward validity or gain.}\\
\endlastfoot
\input{si_full_ernaive.tex}
\end{longtable}
}

{\footnotesize
\begin{longtable}{l cccccc}
\caption{\textbf{Honest-protocol test F1 per method (logistic regression).} Every method's
F1 under the split-then-resample protocol; compare against the trivial-baseline F1 values
in Table~\ref{tab:s3}, or read the per-method deltas directly in Table~\ref{tab:full_df1}.}
\label{tab:full_f1}\\
\toprule
method & sylva & accid & Fram & Pima & thor & fraud \\ \midrule
\endfirsthead
\multicolumn{7}{l}{\footnotesize\emph{Table~\ref{tab:full_f1} continued}}\\
\toprule method & sylva & accid & Fram & Pima & thor & fraud \\ \midrule
\endhead
\bottomrule \multicolumn{7}{r}{\footnotesize continued\ldots}\\
\endfoot
\bottomrule
\multicolumn{7}{@{}p{0.95\textwidth}@{}}{\footnotesize\emph{Codes:} \texttt{deg} no-op (empty synthetic set), \texttt{lib} feature-space-transforming (vote undefined), \texttt{x} run failure, \texttt{--} omitted for runtime (fraud only, 15 slow/backend methods; see justification above). No excluded cell counts toward validity or gain.}\\
\endlastfoot
\input{si_full_benchf1.tex}
\end{longtable}
}

{\footnotesize
\begin{longtable}{l cccccc}
\caption{\textbf{$\Delta$F1 per method vs.\ the best trivial baseline (logistic regression).}
The delta a reader would otherwise compute from Table~\ref{tab:full_f1} and the baselines of
Table~\ref{tab:s3}: method F1 minus the best of no-resample, class-weight and threshold-move
on that dataset. Positive $=$ improvement; the median is below $0.01$ and the only sizeable
entries fall on thoracic surgery, where the unresampled baseline collapses.}
\label{tab:full_df1}\\
\toprule
method & sylva & accid & Fram & Pima & thor & fraud \\ \midrule
\endfirsthead
\multicolumn{7}{l}{\footnotesize\emph{Table~\ref{tab:full_df1} continued}}\\
\toprule method & sylva & accid & Fram & Pima & thor & fraud \\ \midrule
\endhead
\bottomrule \multicolumn{7}{r}{\footnotesize continued\ldots}\\
\endfoot
\bottomrule
\multicolumn{7}{@{}p{0.95\textwidth}@{}}{\footnotesize\emph{Codes:} \texttt{deg} no-op, \texttt{lib} feature-space-transforming, \texttt{x} run failure, \texttt{--} omitted for runtime (fraud only). $\Delta$ is the method's F1 minus the best trivial baseline's on that dataset.}\\
\endlastfoot
\input{si_full_df1.tex}
\end{longtable}
}

\paragraph{Per-method ranking metrics (the threshold-shift claim, made auditable).}
Tables~\ref{tab:full_dauprc}--\ref{tab:full_dauroc} report each method's change in PR-AUC and
ROC-AUC relative to the \emph{best} trivial baseline, under logistic regression---the ranking
counterpart to the F1 table above, and the check a reviewer needs to test the
operating-point-shift claim cell by cell. Across the 515 evaluated method--dataset cells the
median $\Delta$PR-AUC is $-0.005$ and the median $\Delta$ROC-AUC $-0.003$; only 11 and 6 cells
respectively exceed $+0.02$, and those concentrate on thoracic surgery---the one set where the
unresampled baseline degenerates. Where a method raises F1 (Table~\ref{tab:full_f1}) its
ranking almost never follows: the signature of a moved threshold, not new information. A
method that genuinely learned more would show a reproducible positive $\Delta$ here; none does.

{\footnotesize
\begin{longtable}{l cccccc}
\caption{\textbf{$\Delta$PR-AUC per method vs.\ the best trivial baseline (logistic
regression).} Positive $=$ ranking improvement over the best of no-resample, class-weight and
threshold-move on that dataset; near-universally $\le 0$. Read against the F1 gains of
Table~\ref{tab:full_f1}---F1 rises where PR-AUC does not.}
\label{tab:full_dauprc}\\
\toprule
method & sylva & accid & Fram & Pima & thor & fraud \\ \midrule
\endfirsthead
\multicolumn{7}{l}{\footnotesize\emph{Table~\ref{tab:full_dauprc} continued}}\\
\toprule method & sylva & accid & Fram & Pima & thor & fraud \\ \midrule
\endhead
\bottomrule \multicolumn{7}{r}{\footnotesize continued\ldots}\\
\endfoot
\bottomrule
\multicolumn{7}{@{}p{0.95\textwidth}@{}}{\footnotesize\emph{Codes:} \texttt{deg} no-op, \texttt{lib} feature-space-transforming, \texttt{x} run failure, \texttt{--} omitted for runtime (fraud only). $\Delta$ is the method's PR-AUC minus the best trivial baseline's on that dataset.}\\
\endlastfoot
\input{si_full_dauprc.tex}
\end{longtable}
}

{\footnotesize
\begin{longtable}{l cccccc}
\caption{\textbf{$\Delta$ROC-AUC per method vs.\ the best trivial baseline (logistic
regression).} As Table~\ref{tab:full_dauprc}, for ROC-AUC; the largest positive entries sit on
thoracic surgery, where the unresampled baseline collapses.}
\label{tab:full_dauroc}\\
\toprule
method & sylva & accid & Fram & Pima & thor & fraud \\ \midrule
\endfirsthead
\multicolumn{7}{l}{\footnotesize\emph{Table~\ref{tab:full_dauroc} continued}}\\
\toprule method & sylva & accid & Fram & Pima & thor & fraud \\ \midrule
\endhead
\bottomrule \multicolumn{7}{r}{\footnotesize continued\ldots}\\
\endfoot
\bottomrule
\multicolumn{7}{@{}p{0.95\textwidth}@{}}{\footnotesize\emph{Codes:} as Table~\ref{tab:full_dauprc}. $\Delta$ is the method's ROC-AUC minus the best trivial baseline's on that dataset.}\\
\endlastfoot
\input{si_full_dauroc.tex}
\end{longtable}
}

\bibliographystyle{unsrtnat}
\bibliography{refs}

\end{document}

%% file: si_full_ersplit.tex
ADASYN & 0.272 & 0.845 & 0.815 & 0.531 & 0.845 & 0.412 \\
ADG & 0.173 & 0.860 & 0.783 & 0.521 & 0.885 & -- \\
ADOMS & 0.244 & 0.830 & 0.800 & 0.460 & 0.836 & 0.204 \\
AHC & 0.989 & 0.879 & 0.872 & 0.744 & 0.876 & 0.998 \\
AMSCO & 0.219 & 0.853 & 0.804 & 0.485 & 0.846 & -- \\
AND\_SMOTE & 0.201 & 0.824 & 0.779 & 0.409 & 0.778 & 0.039 \\
ANS & 0.206 & 0.848 & 0.800 & 0.541 & 0.762 & 0.078 \\
ASMOBD & deg & deg & 0.742 & 0.574 & 0.765 & 0.018 \\
A\_SUWO & 0.257 & 0.822 & 0.761 & 0.477 & 0.822 & 0.461 \\
Assembled\_SMOTE & 0.196 & 0.837 & 0.795 & 0.490 & 0.824 & 0.177 \\
Borderline\_SMOTE1 & 0.249 & 0.839 & 0.766 & 0.493 & 0.844 & 0.160 \\
Borderline\_SMOTE2 & 0.413 & 0.840 & 0.758 & 0.566 & 0.804 & 0.343 \\
CBSO & 0.222 & 0.841 & 0.810 & 0.502 & 0.793 & 0.333 \\
CCR & 0.360 & 0.831 & 0.778 & 0.504 & 0.833 & 0.195 \\
CE\_SMOTE & 0.211 & 0.834 & 0.806 & 0.462 & 0.827 & 0.220 \\
CURE\_SMOTE & 0.115 & 0.893 & 0.781 & 0.476 & 0.766 & 0.034 \\
DBSMOTE & 0.234 & 0.816 & 0.757 & 0.607 & 0.685 & 0.000 \\
DEAGO & 0.229 & 0.914 & 0.788 & 0.497 & 0.670 & -- \\
DE\_oversampling & 0.258 & 0.836 & 0.788 & 0.487 & 0.812 & 0.239 \\
DSMOTE & 0.138 & 0.846 & 0.711 & 0.224 & 0.830 & 0.686 \\
DSRBF & 0.182 & 0.834 & 0.786 & 0.507 & 0.788 & -- \\
E\_SMOTE & lib & lib & lib & lib & lib & -- \\
Edge\_Det\_SMOTE & 0.191 & 0.845 & 0.800 & 0.488 & 0.811 & 0.115 \\
GASMOTE & 0.211 & 0.840 & 0.820 & 0.475 & 0.826 & -- \\
G\_SMOTE & 0.205 & 0.836 & 0.787 & 0.503 & 0.812 & 0.207 \\
Gaussian\_SMOTE & 0.666 & 0.837 & 0.747 & 0.498 & 0.892 & 0.251 \\
Gazzah & 0.140 & 0.876 & 0.857 & deg & 0.763 & 0.164 \\
IPADE\_ID & deg & deg & deg & deg & deg & -- \\
ISMOTE & deg & deg & deg & deg & deg & deg \\
ISOMAP\_Hybrid & lib & lib & lib & deg & lib & lib \\
KernelADASYN & 0.600 & 0.851 & 0.800 & 0.600 & 0.844 & -- \\
LLE\_SMOTE & 0.116 & 0.883 & 0.785 & 0.409 & 0.729 & 0.121 \\
LN\_SMOTE & 0.176 & 0.823 & 0.755 & 0.421 & 0.821 & 0.063 \\
LVQ\_SMOTE & 0.194 & 0.850 & 0.715 & 0.474 & 0.836 & 0.310 \\
Lee & 0.189 & 0.826 & 0.766 & 0.457 & 0.856 & 0.060 \\
MCT & 0.248 & 0.840 & 0.779 & 0.522 & 0.828 & 0.190 \\
MDO & 0.849 & 0.789 & 0.596 & 0.350 & 0.875 & 0.445 \\
MOT2LD & 0.173 & 0.824 & 0.743 & 0.407 & 0.885 & -- \\
MSMOTE & 0.371 & 0.837 & 0.766 & 0.502 & 0.819 & 0.097 \\
MSYN & 0.115 & 0.816 & 0.758 & 0.424 & 0.803 & x \\
MWMOTE & 0.203 & 0.857 & 0.823 & 0.528 & 0.826 & 0.319 \\
NDO\_sampling & 0.258 & 0.845 & 0.796 & 0.471 & 0.790 & 0.138 \\
NEATER & 0.216 & 0.856 & 0.791 & 0.488 & 0.830 & -- \\
NRAS & 0.195 & 0.785 & 0.728 & 0.417 & 0.882 & 0.051 \\
NRSBoundary\_SMOTE & 0.222 & 0.845 & 0.793 & 0.471 & 0.790 & 0.177 \\
NT\_SMOTE & 0.210 & 0.862 & 0.811 & 0.491 & 0.754 & 0.158 \\
OUPS & 0.218 & 0.854 & 0.784 & 0.559 & 0.810 & 0.304 \\
PDFOS & 0.635 & 0.831 & 0.730 & 0.476 & 0.845 & 0.447 \\
ProWSyn & 0.147 & 0.788 & 0.755 & 0.536 & 0.840 & 0.224 \\
ROSE & 0.267 & 0.829 & 0.776 & 0.510 & 0.855 & 0.356 \\
RWO\_sampling & 0.237 & 0.836 & 0.814 & 0.515 & 0.819 & 0.198 \\
Random\_SMOTE & 0.209 & 0.849 & 0.797 & 0.405 & 0.836 & 0.189 \\
SDSMOTE & 0.216 & 0.829 & 0.795 & 0.478 & 0.840 & 0.176 \\
SL\_graph\_SMOTE & 0.284 & 0.834 & 0.772 & 0.505 & 0.839 & 0.136 \\
SMMO & deg & 0.928 & 0.843 & 0.641 & deg & 0.799 \\
SMOBD & deg & deg & deg & deg & deg & deg \\
SMOTE & 0.215 & 0.841 & 0.790 & 0.472 & 0.795 & 0.183 \\
SMOTEWB & 0.248 & 0.820 & 0.791 & 0.505 & 0.771 & 0.167 \\
SMOTE\_AMSR & 0.207 & 0.858 & 0.803 & 0.559 & 0.795 & 0.291 \\
SMOTE\_Cosine & 0.200 & 0.842 & 0.804 & 0.555 & 0.830 & 0.173 \\
SMOTE\_D & 0.225 & 0.837 & 0.776 & 0.440 & 0.849 & 0.293 \\
SMOTE\_ENN & deg & deg & deg & deg & deg & deg \\
SMOTE\_FRST\_2T & 0.229 & 0.846 & 0.787 & 0.494 & 0.786 & -- \\
SMOTE\_IPF & 0.214 & 0.846 & 0.807 & 0.421 & 0.835 & 0.186 \\
SMOTE\_OUT & 0.211 & 0.846 & 0.802 & 0.462 & 0.852 & 0.155 \\
SMOTE\_PSO & 0.325 & 0.850 & 0.796 & 0.497 & 0.849 & -- \\
SMOTE\_PSOBAT & 0.241 & 0.846 & 0.795 & 0.545 & 0.790 & -- \\
SMOTE\_RSB & deg & deg & 0.630 & 0.240 & 0.813 & 0.011 \\
SMOTE\_TomekLinks & 0.214 & 0.848 & 0.797 & 0.501 & 0.843 & 0.167 \\
SN\_SMOTE & 0.221 & 0.854 & 0.767 & 0.397 & 0.795 & 0.157 \\
SOI\_CJ & 0.243 & 0.826 & 0.743 & 0.534 & deg & 0.052 \\
SOMO & deg & deg & deg & deg & deg & 0.011 \\
SPY & deg & deg & deg & deg & deg & deg \\
SSO & 0.875 & 0.813 & 0.586 & 0.471 & 0.857 & -- \\
SUNDO & 0.986 & 0.869 & 0.856 & deg & 0.857 & 0.999 \\
SVM\_balance & 0.202 & 0.854 & 0.803 & 0.526 & 0.815 & 0.188 \\
SYMPROD & deg & deg & 0.699 & 0.398 & deg & 0.041 \\
Safe\_Level\_SMOTE & 0.297 & 0.866 & 0.793 & 0.490 & 0.812 & 0.146 \\
Selected\_SMOTE & 0.202 & 0.849 & 0.804 & 0.514 & 0.832 & 0.195 \\
Stefanowski & 0.328 & 0.836 & 0.805 & 0.666 & 0.860 & 0.575 \\
Supervised\_SMOTE & 0.191 & 0.834 & 0.718 & 0.493 & 0.773 & -- \\
TRIM\_SMOTE & 0.218 & 0.842 & 0.870 & 0.453 & 0.812 & 0.095 \\
VIS\_RST & 0.283 & deg & deg & deg & deg & 0.107 \\
V\_SYNTH & 0.997 & 0.847 & 0.752 & 0.600 & 0.888 & deg \\
cluster\_SMOTE & 0.210 & 0.847 & 0.770 & 0.450 & 0.850 & 0.248 \\
distance\_SMOTE & 0.174 & 0.860 & 0.789 & 0.460 & 0.811 & 0.159 \\
kmeans\_SMOTE & deg & deg & 0.595 & 0.393 & deg & deg \\
polynom\_fit\_SMOTE\_bus & 0.194 & 0.855 & 0.791 & 0.531 & 0.864 & 0.165 \\
polynom\_fit\_SMOTE\_mesh & 0.193 & 0.861 & 0.799 & 0.526 & 0.832 & 0.173 \\
polynom\_fit\_SMOTE\_poly & 0.044 & 0.919 & 0.838 & 0.688 & 0.830 & 0.012 \\
polynom\_fit\_SMOTE\_star & 0.115 & 0.887 & 0.811 & 0.715 & 0.723 & 0.145 \\

%% file: si_full_ernaive.tex
ADASYN & 0.000 & 0.014 & 0.181 & 0.151 & 0.194 & 0.107 \\
ADG & 0.159 & 0.822 & 0.774 & 0.498 & 0.658 & -- \\
ADOMS & 0.016 & 0.047 & 0.200 & 0.164 & 0.221 & 0.034 \\
AHC & 1.000 & 1.000 & 1.000 & 1.000 & 1.000 & 1.000 \\
AMSCO & 0.000 & 0.013 & 0.209 & 0.139 & 0.232 & -- \\
AND\_SMOTE & 0.000 & 0.002 & 0.094 & 0.103 & 0.103 & 0.002 \\
ANS & 0.000 & 0.149 & 0.415 & 0.375 & 0.242 & 0.025 \\
ASMOBD & deg & deg & 0.052 & 0.099 & 0.297 & 0.000 \\
A\_SUWO & 0.000 & 0.001 & 0.014 & 0.005 & 0.000 & 0.000 \\
Assembled\_SMOTE & 0.000 & 0.011 & 0.147 & 0.086 & 0.139 & 0.047 \\
Borderline\_SMOTE1 & 0.000 & 0.006 & 0.181 & 0.185 & 0.142 & 0.010 \\
Borderline\_SMOTE2 & 0.005 & 0.001 & 0.067 & 0.060 & 0.045 & 0.014 \\
CBSO & 0.003 & 0.107 & 0.406 & 0.190 & 0.288 & 0.111 \\
CCR & 0.000 & 0.094 & 0.021 & 0.007 & 0.304 & 0.009 \\
CE\_SMOTE & 0.000 & 0.004 & 0.153 & 0.082 & 0.158 & 0.068 \\
CURE\_SMOTE & 0.002 & 0.402 & 0.226 & 0.233 & 0.482 & 0.014 \\
DBSMOTE & 0.000 & 0.009 & 0.000 & 0.000 & 0.124 & 0.000 \\
DEAGO & 0.196 & 0.896 & 0.850 & 0.591 & 0.751 & -- \\
DE\_oversampling & 0.008 & 0.028 & 0.104 & 0.065 & 0.089 & 0.000 \\
DSMOTE & 0.051 & 0.859 & 0.824 & 0.190 & 0.918 & 0.656 \\
DSRBF & 0.001 & 0.092 & 0.120 & 0.069 & 0.221 & -- \\
E\_SMOTE & lib & lib & lib & lib & lib & -- \\
Edge\_Det\_SMOTE & 0.000 & 0.013 & 0.158 & 0.125 & 0.194 & 0.018 \\
GASMOTE & 0.000 & 0.005 & 0.158 & 0.105 & 0.148 & -- \\
G\_SMOTE & 0.000 & 0.005 & 0.101 & 0.099 & 0.154 & 0.048 \\
Gaussian\_SMOTE & 0.444 & 0.477 & 0.682 & 0.496 & 0.700 & 0.138 \\
Gazzah & 0.003 & 0.425 & 0.535 & deg & 0.535 & 0.015 \\
IPADE\_ID & deg & deg & deg & deg & deg & -- \\
ISMOTE & deg & deg & deg & deg & deg & deg \\
ISOMAP\_Hybrid & lib & lib & lib & deg & lib & lib \\
KernelADASYN & 0.000 & 0.812 & 0.000 & 1.000 & 0.733 & -- \\
LLE\_SMOTE & 0.061 & 0.622 & 0.714 & 0.353 & 0.694 & 0.089 \\
LN\_SMOTE & 0.012 & 0.040 & 0.076 & 0.099 & 0.030 & 0.002 \\
LVQ\_SMOTE & 0.001 & 0.251 & 0.516 & 0.397 & 0.482 & 0.148 \\
Lee & 0.000 & 0.003 & 0.122 & 0.060 & 0.209 & 0.001 \\
MCT & 0.000 & 0.000 & 0.000 & 0.000 & 0.000 & 0.000 \\
MDO & 0.852 & 0.768 & 0.524 & 0.349 & 0.321 & 0.441 \\
MOT2LD & 0.000 & 0.049 & 0.311 & 0.206 & 0.221 & -- \\
MSMOTE & 0.180 & 0.359 & 0.372 & 0.259 & 0.370 & 0.011 \\
MSYN & 0.000 & 0.002 & 0.072 & 0.066 & 0.052 & x \\
MWMOTE & 0.000 & 0.126 & 0.419 & 0.280 & 0.191 & 0.116 \\
NDO\_sampling & 0.004 & 0.001 & 0.069 & 0.060 & 0.049 & 0.005 \\
NEATER & 0.000 & 0.014 & 0.182 & 0.136 & 0.171 & -- \\
NRAS & 0.040 & 0.208 & 0.167 & 0.160 & 0.248 & 0.011 \\
NRSBoundary\_SMOTE & 0.000 & 0.007 & 0.130 & 0.116 & 0.136 & 0.041 \\
NT\_SMOTE & 0.000 & 0.050 & 0.252 & 0.198 & 0.315 & 0.065 \\
OUPS & 0.094 & 0.504 & 0.610 & 0.552 & 0.567 & 0.172 \\
PDFOS & 0.588 & 0.739 & 0.651 & 0.379 & 0.709 & 0.393 \\
ProWSyn & 0.002 & 0.173 & 0.417 & 0.289 & 0.303 & 0.087 \\
ROSE & 0.084 & 0.274 & 0.573 & 0.478 & 0.509 & 0.276 \\
RWO\_sampling & 0.000 & 0.002 & 0.000 & 0.142 & 0.154 & 0.001 \\
Random\_SMOTE & 0.000 & 0.247 & 0.306 & 0.228 & 0.442 & 0.065 \\
SDSMOTE & 0.000 & 0.008 & 0.163 & 0.082 & 0.197 & 0.036 \\
SL\_graph\_SMOTE & 0.127 & 0.012 & 0.153 & 0.362 & 0.142 & 0.059 \\
SMMO & deg & 0.355 & 0.238 & 0.198 & deg & 0.289 \\
SMOBD & deg & deg & deg & deg & deg & 0.028 \\
SMOTE & 0.000 & 0.009 & 0.155 & 0.082 & 0.197 & 0.036 \\
SMOTEWB & 0.000 & 0.001 & 0.018 & 0.022 & 0.000 & 0.000 \\
SMOTE\_AMSR & 0.002 & 0.487 & 0.675 & 0.422 & 0.430 & 0.090 \\
SMOTE\_Cosine & 0.000 & 0.096 & 0.367 & 0.353 & 0.188 & 0.026 \\
SMOTE\_D & 0.000 & 0.005 & 0.162 & 0.114 & 0.167 & 0.002 \\
SMOTE\_ENN & deg & deg & deg & deg & deg & deg \\
SMOTE\_FRST\_2T & 0.000 & 0.034 & 0.153 & 0.054 & 0.225 & -- \\
SMOTE\_IPF & 0.000 & 0.009 & 0.183 & 0.082 & 0.197 & 0.031 \\
SMOTE\_OUT & 0.001 & 0.006 & 0.101 & 0.052 & 0.167 & 0.011 \\
SMOTE\_PSO & 0.373 & 0.034 & 0.005 & 0.038 & 0.022 & -- \\
SMOTE\_PSOBAT & 0.000 & 0.005 & 0.000 & 0.043 & 0.000 & -- \\
SMOTE\_RSB & deg & deg & 0.296 & 0.000 & 0.083 & 0.000 \\
SMOTE\_TomekLinks & 0.000 & 0.012 & 0.162 & 0.131 & 0.190 & 0.043 \\
SN\_SMOTE & 0.000 & 0.003 & 0.127 & 0.090 & 0.127 & 0.036 \\
SOI\_CJ & 0.060 & 0.010 & 0.128 & 0.129 & deg & 0.028 \\
SOMO & deg & deg & deg & deg & deg & 0.000 \\
SPY & deg & deg & deg & deg & deg & deg \\
SSO & 0.877 & 0.004 & 0.569 & 0.361 & 0.270 & -- \\
SUNDO & 1.000 & 1.000 & 1.000 & deg & 1.000 & 0.998 \\
SVM\_balance & 0.000 & 0.010 & 0.179 & 0.095 & 0.167 & 0.043 \\
SYMPROD & deg & deg & 0.025 & 0.362 & deg & 0.000 \\
Safe\_Level\_SMOTE & 0.135 & 0.635 & 0.577 & 0.332 & 0.648 & 0.072 \\
Selected\_SMOTE & 0.000 & 0.231 & 0.214 & 0.004 & 0.345 & 0.000 \\
Stefanowski & 0.000 & 0.000 & 0.000 & 0.000 & 0.000 & 0.000 \\
Supervised\_SMOTE & 0.001 & 0.052 & 0.099 & 0.336 & 0.027 & -- \\
TRIM\_SMOTE & 0.000 & 0.008 & 0.434 & 0.082 & 0.185 & 0.012 \\
VIS\_RST & 0.064 & deg & deg & deg & deg & 0.031 \\
V\_SYNTH & 1.000 & 0.843 & 0.764 & 0.625 & 0.858 & deg \\
cluster\_SMOTE & 0.000 & 0.005 & 0.116 & 0.095 & 0.088 & 0.074 \\
distance\_SMOTE & 0.001 & 0.100 & 0.238 & 0.155 & 0.291 & 0.054 \\
kmeans\_SMOTE & deg & deg & 0.043 & 0.151 & 0.439 & deg \\
polynom\_fit\_SMOTE\_bus & 0.000 & 0.186 & 0.435 & 0.509 & 0.275 & 0.051 \\
polynom\_fit\_SMOTE\_mesh & 0.004 & 0.370 & 0.636 & 0.448 & 0.430 & 0.085 \\
polynom\_fit\_SMOTE\_poly & 0.028 & 0.926 & 0.824 & 0.763 & 0.782 & 0.000 \\
polynom\_fit\_SMOTE\_star & 0.005 & 0.387 & 0.579 & 0.720 & 0.523 & 0.008 \\

%% file: si_full_benchf1.tex
ADASYN & 0.946 & 0.310 & 0.399 & 0.670 & 0.242 & 0.152 \\
ADG & 0.946 & 0.313 & 0.105 & 0.685 & 0.157 & -- \\
ADOMS & 0.952 & 0.316 & 0.410 & 0.667 & 0.233 & 0.422 \\
AHC & 0.942 & 0.216 & 0.241 & 0.674 & 0.143 & 0.856 \\
AMSCO & 0.950 & 0.263 & 0.392 & 0.704 & 0.243 & -- \\
AND\_SMOTE & 0.946 & 0.319 & 0.417 & 0.674 & 0.138 & 0.724 \\
ANS & 0.945 & 0.315 & 0.370 & 0.659 & 0.143 & 0.721 \\
ASMOBD & 0.946 & 0.197 & 0.281 & 0.674 & 0.129 & 0.882 \\
A\_SUWO & 0.938 & 0.319 & 0.406 & 0.674 & 0.145 & 0.820 \\
Assembled\_SMOTE & 0.947 & 0.316 & 0.405 & 0.663 & 0.254 & 0.418 \\
Borderline\_SMOTE1 & 0.935 & 0.317 & 0.406 & 0.663 & 0.182 & 0.688 \\
Borderline\_SMOTE2 & 0.908 & 0.326 & 0.398 & 0.698 & 0.167 & 0.626 \\
CBSO & 0.948 & 0.317 & 0.392 & 0.700 & 0.246 & 0.179 \\
CCR & 0.941 & 0.318 & 0.396 & 0.692 & 0.257 & 0.424 \\
CE\_SMOTE & 0.952 & 0.317 & 0.412 & 0.663 & 0.226 & 0.428 \\
CURE\_SMOTE & 0.947 & 0.317 & 0.307 & 0.630 & 0.214 & 0.869 \\
DBSMOTE & 0.950 & 0.229 & 0.324 & 0.654 & 0.127 & 0.882 \\
DEAGO & 0.940 & 0.124 & 0.388 & 0.691 & 0.346 & -- \\
DE\_oversampling & 0.943 & 0.321 & 0.401 & 0.697 & 0.279 & 0.415 \\
DSMOTE & 0.901 & 0.108 & 0.317 & 0.623 & 0.145 & 0.738 \\
DSRBF & 0.948 & 0.315 & 0.401 & 0.682 & 0.226 & -- \\
E\_SMOTE & lib & lib & lib & lib & lib & -- \\
Edge\_Det\_SMOTE & 0.946 & 0.313 & 0.421 & 0.682 & 0.215 & 0.556 \\
GASMOTE & 0.947 & 0.304 & 0.408 & 0.682 & 0.145 & -- \\
G\_SMOTE & 0.949 & 0.323 & 0.403 & 0.659 & 0.230 & 0.425 \\
Gaussian\_SMOTE & 0.899 & 0.308 & 0.406 & 0.696 & 0.176 & 0.468 \\
Gazzah & 0.871 & 0.085 & 0.381 & 0.681 & 0.301 & 0.548 \\
IPADE\_ID & 0.849 & 0.180 & 0.355 & 0.674 & 0.164 & -- \\
ISMOTE & 0.819 & 0.284 & 0.405 & 0.684 & 0.317 & 0.163 \\
ISOMAP\_Hybrid & lib & lib & lib & lib & lib & lib \\
KernelADASYN & 0.930 & 0.108 & 0.394 & 0.688 & 0.314 & -- \\
LLE\_SMOTE & 0.942 & 0.327 & 0.422 & 0.671 & 0.281 & 0.374 \\
LN\_SMOTE & 0.948 & 0.321 & 0.416 & 0.674 & 0.148 & 0.738 \\
LVQ\_SMOTE & 0.938 & 0.178 & 0.400 & 0.667 & 0.197 & 0.454 \\
Lee & 0.951 & 0.309 & 0.404 & 0.674 & 0.122 & 0.732 \\
MCT & 0.954 & 0.321 & 0.401 & 0.670 & 0.179 & 0.434 \\
MDO & 0.898 & 0.286 & 0.388 & 0.639 & 0.038 & 0.164 \\
MOT2LD & x & x & x & x & x & -- \\
MSMOTE & 0.896 & 0.319 & 0.407 & 0.670 & 0.140 & 0.722 \\
MSYN & 0.951 & 0.309 & 0.411 & 0.663 & 0.250 & 0.865 \\
MWMOTE & 0.947 & 0.314 & 0.383 & 0.663 & 0.265 & 0.167 \\
NDO\_sampling & 0.912 & 0.315 & 0.411 & 0.663 & 0.188 & 0.505 \\
NEATER & 0.922 & 0.288 & 0.370 & 0.670 & 0.267 & -- \\
NRAS & 0.884 & 0.241 & 0.396 & 0.627 & 0.074 & 0.816 \\
NRSBoundary\_SMOTE & 0.947 & 0.321 & 0.412 & 0.663 & 0.191 & 0.197 \\
NT\_SMOTE & 0.943 & 0.312 & 0.405 & 0.674 & 0.167 & 0.408 \\
OUPS & 0.944 & 0.316 & 0.415 & 0.670 & 0.254 & 0.420 \\
PDFOS & 0.945 & 0.318 & 0.401 & 0.663 & 0.269 & 0.161 \\
ProWSyn & 0.947 & 0.324 & 0.410 & 0.674 & 0.218 & 0.435 \\
ROSE & 0.935 & 0.318 & 0.404 & 0.655 & 0.246 & 0.487 \\
RWO\_sampling & 0.952 & 0.313 & 0.403 & 0.674 & 0.233 & 0.439 \\
Random\_SMOTE & 0.954 & 0.315 & 0.401 & 0.670 & 0.188 & 0.436 \\
SDSMOTE & 0.952 & 0.316 & 0.404 & 0.689 & 0.226 & 0.423 \\
SL\_graph\_SMOTE & 0.908 & 0.317 & 0.406 & 0.663 & 0.121 & 0.547 \\
SMMO & 0.946 & 0.097 & 0.358 & 0.673 & 0.000 & 0.130 \\
SMOBD & 0.946 & 0.097 & 0.058 & 0.629 & 0.000 & 0.877 \\
SMOTE & 0.952 & 0.313 & 0.404 & 0.689 & 0.219 & 0.423 \\
SMOTEWB & 0.950 & 0.324 & 0.413 & 0.674 & 0.254 & 0.437 \\
SMOTE\_AMSR & 0.950 & 0.324 & 0.403 & 0.685 & 0.194 & 0.624 \\
SMOTE\_Cosine & 0.952 & 0.322 & 0.413 & 0.670 & 0.267 & 0.643 \\
SMOTE\_D & 0.947 & 0.306 & 0.425 & 0.682 & 0.230 & 0.606 \\
SMOTE\_ENN & 0.946 & 0.097 & 0.058 & 0.629 & 0.000 & 0.877 \\
SMOTE\_FRST\_2T & 0.950 & 0.301 & 0.396 & 0.659 & 0.145 & -- \\
SMOTE\_IPF & 0.948 & 0.319 & 0.404 & 0.663 & 0.226 & 0.423 \\
SMOTE\_OUT & 0.947 & 0.316 & 0.406 & 0.667 & 0.230 & 0.422 \\
SMOTE\_PSO & 0.954 & 0.261 & 0.310 & 0.685 & 0.278 & -- \\
SMOTE\_PSOBAT & 0.946 & 0.146 & 0.406 & 0.667 & 0.225 & -- \\
SMOTE\_RSB & 0.950 & 0.097 & 0.080 & 0.615 & 0.000 & 0.874 \\
SMOTE\_TomekLinks & 0.948 & 0.312 & 0.414 & 0.701 & 0.197 & 0.423 \\
SN\_SMOTE & 0.947 & 0.324 & 0.405 & 0.659 & 0.254 & 0.412 \\
SOI\_CJ & 0.954 & 0.205 & 0.294 & 0.628 & 0.000 & 0.878 \\
SOMO & 0.942 & 0.103 & 0.252 & 0.671 & 0.121 & 0.857 \\
SPY & 0.884 & 0.135 & 0.154 & 0.692 & 0.062 & 0.857 \\
SSO & 0.937 & 0.309 & 0.394 & 0.638 & 0.191 & -- \\
SUNDO & 0.941 & 0.281 & 0.421 & 0.623 & 0.219 & 0.632 \\
SVM\_balance & 0.950 & 0.314 & 0.389 & 0.702 & 0.244 & 0.401 \\
SYMPROD & 0.946 & 0.097 & 0.294 & 0.658 & 0.046 & 0.875 \\
Safe\_Level\_SMOTE & 0.931 & 0.321 & 0.418 & 0.674 & 0.226 & 0.547 \\
Selected\_SMOTE & 0.944 & 0.317 & 0.398 & 0.670 & 0.219 & 0.438 \\
Stefanowski & 0.956 & 0.277 & 0.348 & 0.670 & 0.269 & 0.839 \\
Supervised\_SMOTE & 0.951 & 0.317 & 0.407 & 0.682 & 0.291 & -- \\
TRIM\_SMOTE & 0.941 & 0.318 & 0.232 & 0.674 & 0.164 & 0.508 \\
VIS\_RST & 0.835 & 0.097 & 0.058 & 0.629 & 0.000 & 0.455 \\
V\_SYNTH & 0.674 & 0.317 & 0.403 & 0.652 & 0.308 & 0.335 \\
cluster\_SMOTE & 0.944 & 0.315 & 0.419 & 0.663 & 0.231 & 0.562 \\
distance\_SMOTE & 0.943 & 0.320 & 0.407 & 0.659 & 0.203 & 0.420 \\
kmeans\_SMOTE & 0.945 & 0.097 & 0.131 & 0.640 & 0.121 & 0.875 \\
polynom\_fit\_SMOTE\_bus & 0.946 & 0.318 & 0.405 & 0.678 & 0.230 & 0.645 \\
polynom\_fit\_SMOTE\_mesh & 0.947 & 0.316 & 0.412 & 0.674 & 0.200 & 0.630 \\
polynom\_fit\_SMOTE\_poly & 0.939 & 0.316 & 0.417 & 0.690 & 0.218 & 0.866 \\
polynom\_fit\_SMOTE\_star & 0.941 & 0.285 & 0.406 & 0.682 & 0.258 & 0.834 \\

%% file: si_full_df1.tex
ADASYN & -0.004 & -0.007 & -0.023 & -0.008 & +0.004 & -0.725 \\
ADG & -0.004 & -0.004 & -0.317 & +0.007 & -0.082 & -- \\
ADOMS & +0.002 & -0.001 & -0.012 & -0.012 & -0.006 & -0.455 \\
AHC & -0.008 & -0.101 & -0.181 & -0.004 & -0.096 & -0.021 \\
AMSCO & 0.000 & -0.054 & -0.030 & +0.026 & +0.004 & -- \\
AND\_SMOTE & -0.004 & +0.002 & -0.005 & -0.004 & -0.101 & -0.153 \\
ANS & -0.005 & -0.002 & -0.052 & -0.019 & -0.096 & -0.156 \\
ASMOBD & -0.004 & -0.120 & -0.141 & -0.004 & -0.110 & +0.005 \\
A\_SUWO & -0.012 & +0.002 & -0.016 & -0.004 & -0.093 & -0.057 \\
Assembled\_SMOTE & -0.002 & -0.001 & -0.017 & -0.015 & +0.015 & -0.459 \\
Borderline\_SMOTE1 & -0.014 & -0.001 & -0.015 & -0.015 & -0.057 & -0.189 \\
Borderline\_SMOTE2 & -0.041 & +0.009 & -0.024 & +0.020 & -0.072 & -0.251 \\
CBSO & -0.002 & 0.000 & -0.029 & +0.021 & +0.007 & -0.698 \\
CCR & -0.009 & +0.001 & -0.025 & +0.014 & +0.018 & -0.453 \\
CE\_SMOTE & +0.002 & 0.000 & -0.010 & -0.015 & -0.013 & -0.449 \\
CURE\_SMOTE & -0.003 & -0.001 & -0.115 & -0.048 & -0.025 & -0.008 \\
DBSMOTE & 0.000 & -0.088 & -0.097 & -0.024 & -0.112 & +0.005 \\
DEAGO & -0.010 & -0.193 & -0.034 & +0.012 & +0.107 & -- \\
DE\_oversampling & -0.007 & +0.004 & -0.020 & +0.018 & +0.040 & -0.462 \\
DSMOTE & -0.049 & -0.209 & -0.105 & -0.056 & -0.093 & -0.139 \\
DSRBF & -0.002 & -0.002 & -0.021 & +0.004 & -0.013 & -- \\
E\_SMOTE & lib & lib & lib & lib & lib & -- \\
Edge\_Det\_SMOTE & -0.004 & -0.004 & -0.001 & +0.004 & -0.023 & -0.321 \\
GASMOTE & -0.003 & -0.013 & -0.014 & +0.004 & -0.093 & -- \\
G\_SMOTE & -0.001 & +0.005 & -0.019 & -0.019 & -0.009 & -0.451 \\
Gaussian\_SMOTE & -0.051 & -0.009 & -0.016 & +0.018 & -0.062 & -0.409 \\
Gazzah & -0.079 & -0.232 & -0.041 & +0.003 & +0.063 & -0.329 \\
IPADE\_ID & -0.101 & -0.137 & -0.067 & -0.004 & -0.075 & -- \\
ISMOTE & -0.131 & -0.033 & -0.016 & +0.006 & +0.078 & -0.713 \\
ISOMAP\_Hybrid & lib & lib & lib & lib & lib & lib \\
KernelADASYN & -0.020 & -0.209 & -0.028 & +0.009 & +0.076 & -- \\
LLE\_SMOTE & -0.007 & +0.009 & +0.001 & -0.008 & +0.042 & -0.503 \\
LN\_SMOTE & -0.002 & +0.004 & -0.006 & -0.004 & -0.091 & -0.138 \\
LVQ\_SMOTE & -0.012 & -0.139 & -0.022 & -0.012 & -0.042 & -0.423 \\
Lee & +0.001 & -0.008 & -0.017 & -0.004 & -0.116 & -0.145 \\
MCT & +0.004 & +0.004 & -0.020 & -0.008 & -0.060 & -0.443 \\
MDO & -0.052 & -0.031 & -0.033 & -0.039 & -0.200 & -0.713 \\
MOT2LD & x & x & x & x & x & -- \\
MSMOTE & -0.054 & +0.002 & -0.015 & -0.008 & -0.098 & -0.155 \\
MSYN & +0.001 & -0.008 & -0.010 & -0.015 & +0.011 & -0.012 \\
MWMOTE & -0.002 & -0.003 & -0.039 & -0.015 & +0.026 & -0.710 \\
NDO\_sampling & -0.037 & -0.002 & -0.010 & -0.015 & -0.051 & -0.372 \\
NEATER & -0.028 & -0.029 & -0.052 & -0.008 & +0.028 & -- \\
NRAS & -0.066 & -0.076 & -0.026 & -0.051 & -0.165 & -0.061 \\
NRSBoundary\_SMOTE & -0.002 & +0.004 & -0.010 & -0.015 & -0.048 & -0.679 \\
NT\_SMOTE & -0.007 & -0.005 & -0.016 & -0.004 & -0.072 & -0.468 \\
OUPS & -0.006 & -0.001 & -0.006 & -0.008 & +0.015 & -0.457 \\
PDFOS & -0.005 & +0.001 & -0.021 & -0.015 & +0.030 & -0.716 \\
ProWSyn & -0.002 & +0.007 & -0.011 & -0.004 & -0.021 & -0.441 \\
ROSE & -0.015 & +0.001 & -0.017 & -0.023 & +0.007 & -0.390 \\
RWO\_sampling & +0.002 & -0.004 & -0.019 & -0.004 & -0.006 & -0.438 \\
Random\_SMOTE & +0.004 & -0.002 & -0.021 & -0.008 & -0.051 & -0.441 \\
SDSMOTE & +0.002 & -0.002 & -0.017 & +0.011 & -0.013 & -0.454 \\
SL\_graph\_SMOTE & -0.041 & -0.001 & -0.015 & -0.015 & -0.118 & -0.330 \\
SMMO & -0.004 & -0.220 & -0.064 & -0.005 & -0.239 & -0.747 \\
SMOBD & -0.004 & -0.220 & -0.364 & -0.049 & -0.239 & 0.000 \\
SMOTE & +0.002 & -0.004 & -0.017 & +0.011 & -0.020 & -0.454 \\
SMOTEWB & 0.000 & +0.007 & -0.009 & -0.004 & +0.015 & -0.440 \\
SMOTE\_AMSR & 0.000 & +0.006 & -0.018 & +0.007 & -0.045 & -0.252 \\
SMOTE\_Cosine & +0.002 & +0.005 & -0.008 & -0.008 & +0.028 & -0.233 \\
SMOTE\_D & -0.003 & -0.011 & +0.004 & +0.004 & -0.009 & -0.271 \\
SMOTE\_ENN & -0.004 & -0.220 & -0.364 & -0.049 & -0.239 & 0.000 \\
SMOTE\_FRST\_2T & 0.000 & -0.016 & -0.026 & -0.019 & -0.094 & -- \\
SMOTE\_IPF & -0.002 & +0.002 & -0.017 & -0.015 & -0.013 & -0.454 \\
SMOTE\_OUT & -0.003 & -0.001 & -0.016 & -0.012 & -0.009 & -0.455 \\
SMOTE\_PSO & +0.004 & -0.056 & -0.111 & +0.007 & +0.039 & -- \\
SMOTE\_PSOBAT & -0.004 & -0.171 & -0.016 & -0.012 & -0.014 & -- \\
SMOTE\_RSB & +0.001 & -0.220 & -0.342 & -0.063 & -0.239 & -0.003 \\
SMOTE\_TomekLinks & -0.002 & -0.005 & -0.008 & +0.023 & -0.042 & -0.454 \\
SN\_SMOTE & -0.002 & +0.007 & -0.017 & -0.019 & +0.015 & -0.465 \\
SOI\_CJ & +0.004 & -0.112 & -0.128 & -0.050 & -0.239 & +0.001 \\
SOMO & -0.007 & -0.214 & -0.170 & -0.008 & -0.118 & -0.020 \\
SPY & -0.066 & -0.182 & -0.268 & +0.014 & -0.176 & -0.020 \\
SSO & -0.013 & -0.008 & -0.028 & -0.040 & -0.048 & -- \\
SUNDO & -0.009 & -0.036 & 0.000 & -0.055 & -0.020 & -0.244 \\
SVM\_balance & 0.000 & -0.004 & -0.033 & +0.024 & +0.005 & -0.476 \\
SYMPROD & -0.004 & -0.220 & -0.128 & -0.020 & -0.192 & -0.002 \\
Safe\_Level\_SMOTE & -0.018 & +0.004 & -0.003 & -0.004 & -0.013 & -0.330 \\
Selected\_SMOTE & -0.005 & 0.000 & -0.024 & -0.008 & -0.020 & -0.439 \\
Stefanowski & +0.006 & -0.041 & -0.073 & -0.008 & +0.030 & -0.038 \\
Supervised\_SMOTE & +0.001 & 0.000 & -0.014 & +0.004 & +0.052 & -- \\
TRIM\_SMOTE & -0.009 & +0.001 & -0.190 & -0.004 & -0.075 & -0.369 \\
VIS\_RST & -0.115 & -0.220 & -0.364 & -0.049 & -0.239 & -0.422 \\
V\_SYNTH & -0.276 & 0.000 & -0.019 & -0.026 & +0.069 & -0.542 \\
cluster\_SMOTE & -0.006 & -0.002 & -0.002 & -0.015 & -0.008 & -0.315 \\
distance\_SMOTE & -0.007 & +0.003 & -0.014 & -0.019 & -0.035 & -0.457 \\
kmeans\_SMOTE & -0.005 & -0.220 & -0.291 & -0.038 & -0.118 & -0.002 \\
polynom\_fit\_SMOTE\_bus & -0.004 & +0.001 & -0.017 & 0.000 & -0.009 & -0.232 \\
polynom\_fit\_SMOTE\_mesh & -0.002 & -0.002 & -0.010 & -0.004 & -0.039 & -0.246 \\
polynom\_fit\_SMOTE\_poly & -0.011 & -0.001 & -0.005 & +0.012 & -0.021 & -0.011 \\
polynom\_fit\_SMOTE\_star & -0.009 & -0.032 & -0.015 & +0.004 & +0.019 & -0.042 \\

%% file: si_full_dauprc.tex
ADASYN & -0.004 & -0.010 & -0.023 & 0.000 & -0.012 & -0.027 \\
ADG & -0.003 & -0.016 & -0.176 & -0.010 & -0.017 & -- \\
ADOMS & +0.001 & -0.007 & -0.017 & -0.009 & -0.023 & -0.018 \\
AHC & -0.003 & +0.001 & -0.005 & -0.002 & -0.011 & -0.002 \\
AMSCO & -0.001 & -0.013 & -0.040 & +0.002 & -0.026 & -- \\
AND\_SMOTE & -0.001 & -0.006 & -0.025 & -0.003 & -0.033 & -0.014 \\
ANS & -0.002 & +0.001 & -0.063 & -0.008 & -0.048 & +0.001 \\
ASMOBD & -0.003 & -0.095 & -0.131 & +0.009 & -0.043 & 0.000 \\
A\_SUWO & -0.008 & -0.007 & -0.023 & -0.001 & -0.042 & -0.032 \\
Assembled\_SMOTE & -0.001 & -0.006 & -0.032 & +0.003 & -0.015 & 0.000 \\
Borderline\_SMOTE1 & -0.005 & -0.006 & -0.035 & -0.008 & -0.046 & -0.029 \\
Borderline\_SMOTE2 & -0.012 & -0.005 & -0.034 & +0.004 & -0.043 & -0.022 \\
CBSO & -0.005 & -0.005 & -0.008 & 0.000 & +0.006 & -0.018 \\
CCR & +0.001 & -0.004 & -0.007 & -0.004 & -0.006 & -0.012 \\
CE\_SMOTE & -0.002 & -0.007 & -0.028 & +0.002 & -0.025 & -0.001 \\
CURE\_SMOTE & -0.004 & -0.006 & -0.015 & +0.002 & -0.025 & +0.003 \\
DBSMOTE & -0.001 & -0.089 & -0.107 & -0.029 & -0.045 & +0.001 \\
DEAGO & -0.008 & 0.000 & -0.008 & +0.005 & +0.049 & -- \\
DE\_oversampling & +0.001 & -0.008 & -0.003 & +0.008 & +0.015 & -0.038 \\
DSMOTE & -0.012 & -0.006 & -0.054 & -0.015 & -0.037 & -0.020 \\
DSRBF & -0.003 & -0.013 & -0.017 & -0.004 & -0.013 & -- \\
E\_SMOTE & lib & lib & lib & lib & lib & -- \\
Edge\_Det\_SMOTE & -0.002 & -0.012 & -0.024 & +0.002 & -0.016 & -0.002 \\
GASMOTE & -0.003 & -0.013 & -0.010 & -0.012 & -0.036 & -- \\
G\_SMOTE & -0.003 & -0.006 & -0.010 & -0.004 & -0.021 & -0.001 \\
Gaussian\_SMOTE & +0.006 & -0.019 & -0.005 & +0.011 & -0.032 & -0.014 \\
Gazzah & -0.051 & -0.072 & -0.046 & -0.038 & +0.035 & -0.061 \\
IPADE\_ID & -0.049 & -0.037 & -0.068 & -0.003 & -0.043 & -- \\
ISMOTE & -0.058 & -0.052 & -0.013 & -0.003 & +0.055 & -0.432 \\
ISOMAP\_Hybrid & lib & lib & lib & lib & lib & lib \\
KernelADASYN & -0.005 & -0.063 & -0.033 & -0.011 & +0.009 & -- \\
LLE\_SMOTE & -0.004 & -0.020 & +0.015 & +0.005 & +0.009 & -0.013 \\
LN\_SMOTE & -0.001 & -0.005 & -0.020 & 0.000 & -0.036 & -0.010 \\
LVQ\_SMOTE & -0.008 & -0.064 & -0.008 & -0.006 & -0.026 & -0.080 \\
Lee & -0.001 & -0.018 & -0.021 & +0.001 & -0.034 & -0.010 \\
MCT & -0.001 & +0.001 & +0.005 & -0.001 & -0.024 & -0.001 \\
MDO & -0.023 & -0.048 & -0.057 & +0.005 & -0.063 & -0.309 \\
MOT2LD & x & x & x & x & x & -- \\
MSMOTE & -0.026 & -0.004 & -0.021 & +0.009 & -0.026 & -0.012 \\
MSYN & -0.003 & -0.011 & +0.027 & -0.004 & -0.020 & +0.006 \\
MWMOTE & -0.003 & -0.002 & -0.047 & -0.001 & -0.027 & -0.018 \\
NDO\_sampling & -0.008 & -0.011 & -0.036 & -0.010 & -0.007 & -0.005 \\
NEATER & -0.051 & -0.012 & -0.032 & -0.003 & -0.021 & -- \\
NRAS & -0.015 & -0.073 & -0.031 & +0.004 & -0.035 & -0.007 \\
NRSBoundary\_SMOTE & -0.006 & -0.009 & -0.021 & -0.003 & -0.027 & -0.028 \\
NT\_SMOTE & -0.003 & -0.018 & -0.025 & +0.001 & -0.038 & +0.001 \\
OUPS & -0.003 & -0.002 & -0.008 & +0.003 & -0.001 & -0.001 \\
PDFOS & -0.003 & -0.023 & -0.033 & +0.006 & -0.014 & -0.380 \\
ProWSyn & -0.004 & +0.002 & +0.015 & -0.005 & -0.013 & 0.000 \\
ROSE & +0.001 & -0.012 & -0.017 & -0.011 & -0.014 & -0.024 \\
RWO\_sampling & 0.000 & -0.005 & +0.005 & -0.001 & -0.010 & -0.003 \\
Random\_SMOTE & +0.002 & -0.008 & -0.026 & +0.001 & -0.037 & -0.002 \\
SDSMOTE & -0.002 & -0.010 & -0.020 & +0.004 & -0.015 & +0.001 \\
SL\_graph\_SMOTE & -0.006 & -0.006 & -0.035 & -0.002 & -0.052 & -0.022 \\
SMMO & -0.003 & 0.000 & -0.070 & -0.017 & 0.000 & -0.049 \\
SMOBD & -0.003 & 0.000 & 0.000 & 0.000 & 0.000 & 0.000 \\
SMOTE & -0.003 & -0.007 & -0.020 & +0.003 & -0.029 & +0.001 \\
SMOTEWB & -0.002 & -0.001 & -0.036 & -0.001 & -0.011 & -0.002 \\
SMOTE\_AMSR & +0.002 & -0.005 & -0.015 & -0.002 & -0.026 & -0.008 \\
SMOTE\_Cosine & 0.000 & +0.002 & -0.003 & 0.000 & -0.006 & +0.002 \\
SMOTE\_D & -0.004 & -0.019 & +0.022 & +0.004 & -0.036 & -0.003 \\
SMOTE\_ENN & -0.003 & 0.000 & 0.000 & 0.000 & 0.000 & 0.000 \\
SMOTE\_FRST\_2T & -0.001 & -0.036 & -0.022 & -0.003 & -0.045 & -- \\
SMOTE\_IPF & -0.002 & -0.009 & -0.020 & -0.006 & -0.020 & +0.001 \\
SMOTE\_OUT & -0.001 & -0.012 & -0.018 & -0.003 & -0.034 & -0.001 \\
SMOTE\_PSO & -0.004 & -0.015 & -0.033 & -0.001 & -0.004 & -- \\
SMOTE\_PSOBAT & -0.003 & -0.002 & +0.002 & +0.003 & -0.027 & -- \\
SMOTE\_RSB & -0.003 & 0.000 & +0.006 & 0.000 & +0.015 & 0.000 \\
SMOTE\_TomekLinks & -0.002 & -0.008 & -0.018 & +0.005 & -0.032 & +0.001 \\
SN\_SMOTE & -0.002 & -0.011 & -0.011 & -0.002 & -0.022 & 0.000 \\
SOI\_CJ & -0.003 & -0.063 & -0.031 & -0.006 & 0.000 & -0.003 \\
SOMO & -0.005 & -0.006 & -0.086 & -0.003 & +0.034 & -0.004 \\
SPY & -0.027 & 0.000 & -0.002 & +0.010 & +0.005 & +0.003 \\
SSO & +0.004 & -0.014 & -0.007 & -0.002 & +0.033 & -- \\
SUNDO & +0.002 & -0.050 & -0.030 & -0.015 & -0.035 & -0.017 \\
SVM\_balance & -0.003 & -0.008 & -0.049 & +0.011 & -0.030 & +0.002 \\
SYMPROD & -0.003 & 0.000 & -0.058 & -0.001 & -0.036 & +0.002 \\
Safe\_Level\_SMOTE & -0.002 & -0.005 & +0.012 & -0.004 & -0.023 & -0.023 \\
Selected\_SMOTE & +0.001 & -0.010 & -0.021 & +0.003 & -0.027 & 0.000 \\
Stefanowski & -0.005 & -0.002 & -0.008 & -0.011 & +0.028 & +0.004 \\
Supervised\_SMOTE & -0.001 & -0.003 & +0.012 & -0.001 & -0.001 & -- \\
TRIM\_SMOTE & -0.002 & -0.018 & -0.193 & -0.003 & -0.041 & -0.003 \\
VIS\_RST & -0.075 & 0.000 & 0.000 & 0.000 & 0.000 & -0.018 \\
V\_SYNTH & -0.018 & -0.002 & +0.007 & -0.013 & +0.041 & -0.068 \\
cluster\_SMOTE & -0.003 & -0.001 & -0.003 & +0.003 & -0.025 & +0.001 \\
distance\_SMOTE & -0.003 & -0.010 & -0.022 & -0.004 & -0.035 & +0.001 \\
kmeans\_SMOTE & -0.004 & 0.000 & -0.014 & +0.004 & +0.031 & -0.002 \\
polynom\_fit\_SMOTE\_bus & -0.003 & 0.000 & +0.010 & -0.004 & -0.006 & +0.002 \\
polynom\_fit\_SMOTE\_mesh & -0.003 & -0.003 & +0.015 & +0.002 & -0.009 & +0.002 \\
polynom\_fit\_SMOTE\_poly & -0.007 & -0.007 & -0.009 & +0.002 & +0.023 & -0.006 \\
polynom\_fit\_SMOTE\_star & -0.004 & +0.002 & +0.007 & -0.001 & +0.005 & +0.004 \\

%% file: si_full_dauroc.tex
ADASYN & 0.000 & -0.013 & -0.003 & -0.001 & -0.026 & -0.022 \\
ADG & 0.000 & -0.018 & -0.267 & -0.003 & -0.069 & -- \\
ADOMS & 0.000 & -0.005 & -0.005 & -0.005 & -0.041 & +0.001 \\
AHC & 0.000 & -0.005 & -0.003 & -0.002 & -0.022 & -0.001 \\
AMSCO & 0.000 & -0.014 & -0.017 & +0.001 & -0.058 & -- \\
AND\_SMOTE & 0.000 & -0.015 & -0.004 & -0.003 & -0.075 & -0.017 \\
ANS & 0.000 & -0.006 & -0.048 & -0.006 & -0.131 & -0.001 \\
ASMOBD & 0.000 & -0.112 & -0.129 & +0.004 & -0.151 & 0.000 \\
A\_SUWO & -0.001 & -0.013 & -0.032 & +0.001 & -0.119 & -0.026 \\
Assembled\_SMOTE & 0.000 & -0.008 & -0.006 & 0.000 & -0.020 & -0.001 \\
Borderline\_SMOTE1 & 0.000 & -0.016 & -0.016 & -0.007 & -0.125 & -0.023 \\
Borderline\_SMOTE2 & -0.001 & -0.010 & -0.019 & +0.005 & -0.110 & -0.022 \\
CBSO & 0.000 & 0.000 & -0.001 & +0.003 & +0.001 & -0.008 \\
CCR & 0.000 & +0.001 & 0.000 & +0.001 & -0.016 & +0.001 \\
CE\_SMOTE & 0.000 & -0.014 & -0.005 & -0.001 & -0.060 & -0.002 \\
CURE\_SMOTE & 0.000 & -0.006 & -0.019 & -0.008 & -0.059 & -0.001 \\
DBSMOTE & 0.000 & -0.101 & -0.123 & -0.005 & -0.154 & 0.000 \\
DEAGO & -0.004 & -0.003 & -0.020 & +0.010 & +0.077 & -- \\
DE\_oversampling & 0.000 & -0.003 & -0.001 & +0.006 & +0.028 & -0.012 \\
DSMOTE & -0.001 & -0.005 & -0.029 & -0.002 & -0.155 & -0.016 \\
DSRBF & 0.000 & -0.014 & -0.002 & -0.002 & -0.020 & -- \\
E\_SMOTE & lib & lib & lib & lib & lib & -- \\
Edge\_Det\_SMOTE & 0.000 & -0.017 & -0.004 & +0.002 & -0.047 & +0.002 \\
GASMOTE & 0.000 & -0.017 & 0.000 & -0.010 & -0.088 & -- \\
G\_SMOTE & 0.000 & -0.013 & +0.001 & -0.005 & -0.050 & -0.005 \\
Gaussian\_SMOTE & 0.000 & -0.019 & -0.003 & +0.009 & -0.093 & +0.004 \\
Gazzah & -0.003 & -0.031 & -0.031 & -0.004 & +0.043 & -0.005 \\
IPADE\_ID & -0.004 & -0.042 & -0.040 & +0.004 & -0.117 & -- \\
ISMOTE & -0.003 & -0.047 & -0.023 & +0.004 & +0.066 & -0.010 \\
ISOMAP\_Hybrid & lib & lib & lib & lib & lib & lib \\
KernelADASYN & 0.000 & -0.050 & -0.020 & +0.004 & +0.005 & -- \\
LLE\_SMOTE & 0.000 & -0.008 & +0.004 & 0.000 & +0.019 & -0.001 \\
LN\_SMOTE & 0.000 & -0.012 & -0.012 & +0.002 & -0.084 & -0.013 \\
LVQ\_SMOTE & -0.001 & -0.102 & -0.013 & -0.003 & -0.054 & -0.039 \\
Lee & 0.000 & -0.019 & -0.012 & +0.002 & -0.073 & -0.014 \\
MCT & 0.000 & +0.004 & +0.001 & +0.001 & -0.051 & 0.000 \\
MDO & -0.001 & -0.055 & -0.028 & +0.003 & -0.192 & -0.033 \\
MOT2LD & x & x & x & x & x & -- \\
MSMOTE & -0.001 & -0.012 & -0.013 & +0.004 & -0.064 & -0.016 \\
MSYN & -0.001 & -0.012 & +0.004 & -0.001 & -0.041 & 0.000 \\
MWMOTE & 0.000 & 0.000 & -0.029 & -0.004 & -0.056 & -0.008 \\
NDO\_sampling & 0.000 & -0.009 & -0.009 & -0.003 & -0.002 & -0.001 \\
NEATER & -0.002 & -0.011 & -0.005 & -0.003 & -0.037 & -- \\
NRAS & -0.001 & -0.086 & -0.021 & -0.001 & -0.132 & -0.009 \\
NRSBoundary\_SMOTE & 0.000 & -0.015 & -0.002 & -0.003 & -0.061 & -0.017 \\
NT\_SMOTE & 0.000 & -0.021 & -0.005 & 0.000 & -0.083 & -0.001 \\
OUPS & 0.000 & -0.007 & -0.004 & 0.000 & -0.010 & 0.000 \\
PDFOS & -0.001 & -0.008 & -0.006 & +0.005 & -0.033 & -0.025 \\
ProWSyn & 0.000 & -0.001 & -0.004 & +0.003 & -0.021 & 0.000 \\
ROSE & 0.000 & -0.003 & -0.004 & -0.006 & -0.064 & +0.004 \\
RWO\_sampling & 0.000 & -0.001 & +0.001 & +0.002 & -0.019 & 0.000 \\
Random\_SMOTE & 0.000 & -0.013 & -0.008 & 0.000 & -0.087 & +0.001 \\
SDSMOTE & 0.000 & -0.013 & -0.002 & +0.002 & -0.033 & 0.000 \\
SL\_graph\_SMOTE & 0.000 & -0.016 & -0.016 & +0.001 & -0.149 & -0.017 \\
SMMO & 0.000 & -0.003 & -0.060 & -0.015 & 0.000 & -0.025 \\
SMOBD & 0.000 & -0.003 & -0.002 & 0.000 & 0.000 & 0.000 \\
SMOTE & 0.000 & -0.014 & -0.002 & 0.000 & -0.068 & 0.000 \\
SMOTEWB & 0.000 & +0.005 & -0.023 & 0.000 & -0.026 & 0.000 \\
SMOTE\_AMSR & 0.000 & -0.008 & -0.002 & 0.000 & -0.054 & -0.001 \\
SMOTE\_Cosine & 0.000 & +0.002 & 0.000 & 0.000 & -0.016 & +0.001 \\
SMOTE\_D & 0.000 & -0.022 & +0.007 & +0.001 & -0.083 & -0.002 \\
SMOTE\_ENN & 0.000 & -0.003 & -0.002 & 0.000 & 0.000 & 0.000 \\
SMOTE\_FRST\_2T & 0.000 & -0.038 & -0.005 & -0.004 & -0.113 & -- \\
SMOTE\_IPF & 0.000 & -0.012 & -0.002 & -0.006 & -0.041 & 0.000 \\
SMOTE\_OUT & 0.000 & -0.022 & -0.001 & -0.003 & -0.074 & -0.001 \\
SMOTE\_PSO & 0.000 & -0.007 & -0.004 & +0.003 & -0.015 & -- \\
SMOTE\_PSOBAT & 0.000 & -0.008 & -0.001 & 0.000 & -0.056 & -- \\
SMOTE\_RSB & 0.000 & -0.003 & -0.002 & +0.001 & +0.032 & -0.001 \\
SMOTE\_TomekLinks & 0.000 & -0.010 & -0.005 & +0.003 & -0.078 & 0.000 \\
SN\_SMOTE & 0.000 & -0.013 & -0.002 & -0.003 & -0.045 & -0.001 \\
SOI\_CJ & 0.000 & -0.068 & -0.029 & -0.014 & 0.000 & 0.000 \\
SOMO & 0.000 & -0.002 & -0.101 & -0.004 & +0.007 & -0.004 \\
SPY & -0.002 & -0.003 & -0.004 & +0.004 & 0.000 & 0.000 \\
SSO & 0.000 & -0.008 & -0.012 & -0.004 & -0.040 & -- \\
SUNDO & 0.000 & -0.066 & -0.007 & -0.012 & -0.114 & +0.002 \\
SVM\_balance & 0.000 & -0.009 & -0.016 & +0.011 & -0.066 & 0.000 \\
SYMPROD & 0.000 & -0.003 & -0.039 & -0.005 & -0.087 & 0.000 \\
Safe\_Level\_SMOTE & 0.000 & -0.012 & +0.001 & -0.003 & -0.059 & -0.017 \\
Selected\_SMOTE & 0.000 & -0.012 & -0.006 & -0.001 & -0.054 & -0.002 \\
Stefanowski & 0.000 & +0.001 & -0.003 & -0.007 & +0.017 & -0.001 \\
Supervised\_SMOTE & 0.000 & +0.002 & -0.001 & 0.000 & +0.007 & -- \\
TRIM\_SMOTE & 0.000 & -0.018 & -0.238 & -0.003 & -0.095 & +0.002 \\
VIS\_RST & -0.004 & -0.003 & -0.002 & 0.000 & 0.000 & -0.027 \\
V\_SYNTH & -0.001 & -0.003 & -0.006 & -0.011 & +0.054 & -0.014 \\
cluster\_SMOTE & 0.000 & -0.012 & -0.001 & 0.000 & -0.063 & 0.000 \\
distance\_SMOTE & 0.000 & -0.011 & -0.008 & -0.003 & -0.084 & -0.001 \\
kmeans\_SMOTE & 0.000 & -0.003 & -0.007 & +0.002 & +0.012 & -0.001 \\
polynom\_fit\_SMOTE\_bus & 0.000 & 0.000 & 0.000 & -0.002 & -0.009 & -0.005 \\
polynom\_fit\_SMOTE\_mesh & 0.000 & 0.000 & +0.001 & 0.000 & -0.033 & -0.002 \\
polynom\_fit\_SMOTE\_poly & -0.001 & +0.001 & -0.003 & +0.003 & +0.011 & 0.000 \\
polynom\_fit\_SMOTE\_star & 0.000 & 0.000 & -0.002 & -0.001 & +0.007 & -0.002 \\